\documentclass[11pt,twoside]{article}
\usepackage{customize}

\title{Noise Estimation in Gaussian Process Regression}

\author{Siavash Ameli\protect\thanks{Email address: \href{mailto:sameli@berkeley.edu}{\protect\nolinkurl{sameli@berkeley.edu}}}\;}
\author{Shawn C. Shadden\protect\thanks{Email address: \href{mailto:shadden@berkeley.edu}{\protect\nolinkurl{shadden@berkeley.edu}}}}

\affil{\small\textit{Mechanical Engineering}, \small\textit{University of California}, \small\textit{Berkeley, CA, USA 94720}}

\date{}

\ifpdf
\hypersetup{
  pdftitle={Noise Estimation in Gaussian Process Regression},
  pdfauthor={S. Ameli and S. C. Shadden}
}
\fi

\def \X {X}
\def \eigM {\psi}
\def \basis {\phi}

\begin{document}

\maketitle

\begin{abstract}
    We develop a computational procedure to estimate the covariance hyperparameters for semiparametric Gaussian process regression models with additive noise. Namely, the presented method can be used to efficiently estimate the variance of the correlated error, and the variance of the noise based on maximizing a marginal likelihood function. Our method involves suitably reducing the dimensionality of the hyperparameter space to simplify the estimation procedure to a univariate root-finding problem. Moreover, we derive bounds and asymptotes of the marginal likelihood function and its derivatives, which are useful to narrowing the initial range of the hyperparameter search. Using numerical examples, we demonstrate the computational advantages and robustness of the presented approach compared to traditional parameter optimization.

\paragraph{Keywords:} parameter estimation, mixed covariance, multivariable linear model, correlated error, nugget

\paragraph{Mathematics Subject Classification (2020):} 62J05, 62H12
\end{abstract}


\section{Introduction}

Gaussian process models are commonly used in a wide range of statistical inference and machine learning methods such as regression and classification \citep{NEAL-1998,MACKAY-1998,SEEGER-2004,RASMUSSEN-2006}, latent variable models \citep{LAWRENCE-2003}, neural networks \citep{NEAL-1996}, and deep belief networks \citep{DAMIANOU-2013}.  

An important application of the Gaussian process is in data regression. In linear regression models, using a Gaussian process prior, a stochastic function \(z(\vect{x})\) is modeled by 
\( z(\vect{x}) = \mu(\vect{x}) + \delta(\vect{x}) \),
where \(\mu\) represents the trend of the data and \(\delta\) represents residual error, which is assumed to be a Gaussian process, i.e. the joint distribution between any finite set of points is normal. 

A Gaussian process is characterized by a mean function \(\mu(\vect{x}) = \mathbb{E}(z(\vect{x}))\) and a covariance function \(\Sigma(\vect{x},\vect{x}') = \operatorname{cov}(z(\vect{x}),z(\vect{x}'))\). The covariance function \(\Sigma\) is often assumed to be of the form \(\sigma^2 K(\vect{x},\vect{x}'|\vect{\theta})\), where \(\sigma^2\) is the variance of the stationary residual error, and \(K\) is a continuous correlation function depending on some hyperparameters \(\vect{\theta}\), where the correlation represents the variability of \(z\) with \(\vect{x}\). 

In many models, the Gaussian process is considered together with an uncorrelated additive noise to either incorporate the uncertainty of the data or to regularize ill-conditioned problems. In such scenarios, the covariance function includes an additional covariance representing the local noise, i.e.,
\begin{equation} 
    \Sigma(\vect{x},\vect{x}'|\vect{\theta}) = \sigma^2 K(\vect{x},\vect{x}'|\vect{\theta}) + \sigma_0^2 I(\vect{x},\vect{x}'), \label{eq:covariance-function}
\end{equation}
where \(I\) is a discontinuous (Kronecker delta) function equal to \(1\) if \(\vect{x} = \vect{x}'\) and zero otherwise, and \(\sigma_0^2\) is the noise variance. Mixed covariance models analogous to \eqref{eq:covariance-function} have been studied in many applications, such as in spatial linear models in geostatistics \cite[p. 59]{CRESSIE-1993}, \cite[\S 3.7]{STEIN-1999}, \cite[\S 2.3 and 23.3]{GELFAND-2010}, \citep{KLEIBER-2012}, \cite[Ch. 8]{WACKERNAGEL-2013}, \citep{GENTON-2015}, in computer experiments where the uncertainty of simulations is modeled by noise in emulators \citep{OHAGAN-2001}, \citep{GRAMACY-2012}, \citep{ANDRIANAKIS-2012}, deterministic computer experiments \citep{PEPELYSHEV-2012}, \citep{PENG-2014}, \citep{LEE-2018}, and in optimal designs of experiments \citep{STEIN-2005}. 

Once a covariance structure is assumed, e.g., \eqref{eq:covariance-function}, the primary task in regression is \emph{model selection}, i.e., to learn the model hyperparameters, e.g., \((\sigma^2,\sigma_0^2,\vect{\theta})\), from the data. Various model selection criteria exist, such as maximum likelihood estimation (MLE) methods \citep{WILLIAMS-1996}, cross-validation methods \citep{SUNDARARAJAN-2001}, and expectation-maximization for Gaussian process latent variable models and probabilistic principal component analysis \citep{TIPPING-1999}, \citep[\S 12.2]{BISHOP-2006}. Regardless of the method, however, the estimation of \((\sigma^2,\sigma_0^2,\vect{\theta})\) altogether is unwieldy; the hyperparameter space is too large for non-convex global optimization, and gradient-based local optimization usually fails to converge to proper hyperparameter values \citep[\S 3.3]{DALLAIRE-2009}. While there are several works on sparse approximation techniques for Gaussian process regression \citep{SMOLA-2001,LAWRENCE-2002,SEEGER-2003,QUINONERO-2005,SNELSON-2006,TITSIAS-2009,DEISENROTH-2015}, hyperparameter estimation usually remains a significant challenge whether using a full or sparse method.

Recognizing that the variances \(\sigma^2\) and \(\sigma_0^2\) play a crucial role (e.g., in model fit, regularization, and overall interpretation of the data), it is generally desirable to first estimate \((\sigma^2,\sigma_0^2)\) for a given hyperparameter set \(\vect{\theta}\).
Optimization of \(\vect{\theta}\) can then be achieved by \emph{profiling out} the estimated variances \((\sigma^2,\sigma_0^2)\) from the model selection criterion.

Estimating \(\sigma^2\) and \(\sigma_0^2\) in \eqref{eq:covariance-function} is analogous to variance estimation in linear mixed models of hierarchical data. Among the earlier attempts to find covariance parameters of mixed models, \citet[\S 4 and \S 5]{HARTLEY-1967} applied the steepest ascent method to maximize likelihood using the derivatives of the likelihood function with respect to both the variance of the residual error and the ratio of the two variances as a new parameter. \citet{PATTERSON-1971} maximized likelihood over a set of selected contrast, which is known as restricted MLE. Another notable technique is the Minimum Norm Quadratic Unbiased Estimation (MINQUE) by \cite{RAO-1971a,RAO-1971b,RAO-1972,RAO-1979} for the estimation of covariance components, which is computationally less expensive than MLE but more restrictive. In another approach, \citet{LINDSTORM-1988} applied the expectation-maximization method to estimate the covariance parameter of mixed covariance models. Reviews of classical estimation methods and variance analyses are described, for instance, in \cite{HARVILLE-1977}, \citet{KITANIDIS-1987} and \citet{SEARLE-1971,SEARLE-1995}. 

While any of the above methods can be applied for variance estimation, we show herein that the variance estimation problem can be effectively simplified to a single hyperparameter estimation by a suitable reformulation of the MLE problem, and this significantly reduces the computational cost \emph{and} improves the convergence of the estimation procedure. 
Our developments can be summarized by the following main contributions:



\begin{itemize}


    \item We consider a linear model for the mean \(\mu\). The mean function has largely been ignored in the literature as it leads to lengthy formulations for the derivatives of the marginal likelihood function (MLF) and subsequent analysis. By defining suitable variables for our model, we will show in \Cref{prop:der-L} that the derivatives of the MLF with respect to an arbitrary hyperparameter ---later assumed to be \((\sigma^2,\sigma_0^2)\)--- can be represented in a tractable formulation. 

    \item In a conventional gradient-based approach, the model hyperparameters are estimated from the equations describing the derivatives of the MLF using iterative numerical schemes. However, this is computationally expensive (as we will show). In \Cref{thm:eta} we reduce the formulation of the set of derivatives for the two hyperparameters to a univariate equation that depends only on the ratio of the noise variance over the error variance as a new hyperparameter. This enables the problem to be vastly simplified from a multi-dimensional optimization to a univariate root-finding problem.


    \item In \Cref{prop:der-bound}, we derive bounds for the derivatives of the MLF, and demonstrate how these can assist the numerical optimization. Moreover, in \Cref{prop:asymptote}, we derive an asymptotic approximation of the derivative of the MLF, which can be used to provide a rough estimate of hyperparameters to initialize the root-finding procedure.
\end{itemize}

We demonstrate the computational advantage of our method through some examples, and we accomplish the following results:

\begin{itemize}

    \item \emph{Performance.} Compared to the traditional hyperparameter optimization with similar scalability  ---\(\mathcal{O}(n^{2.5})\) and \(\mathcal{O}(n^2)\) respectively for dense and sparse correlation matrices--- the computational cost of our method is reduced by orders of magnitude. 
    
    
    \item \emph{Robustness.} In contrast to traditional hyperparameter optimization, we demonstrate the insensitivity of the presented method to the initial hyperparameter guess, and that the optimal solution can be found by a local search with significantly fewer iterations.
\end{itemize}

The main methodology is presented in \S \ref{sec:complete-estimate} where we derive derivatives of the MLF and reduce the dimension of the space of the hyperparameter search. In \S \ref{sec:properties}, we derive properties of the MLF and its derivatives---namely bounds and an asymptotic relation. An implementation of our algorithm is given in \S\ref{sec:efficient}. We present numerical experiments in \S\ref{sec:example} to demonstrate the efficiency and flexibility of the method. A summary of the work is given in \S \ref{sec:conclusion}. 
A Python program to reproduce the results in this paper can be found at \url{https://github.com/ameli/glearn}.


\section{Gaussian Process, a Background} \label{sec:SLR}

In \S \ref{sec:model}, we more fully describe the components of our Gaussian process model. In \S \ref{sec:prelim}, we introduce the problem of parameter estimation and present the least squares solution to this estimation problem in the simplified cases where the error variance or noise variance is known.


\subsection{Model Description} \label{sec:model}

Let \(z: \mathcal{D} \to \mathbb{R}\) be a stochastic function in the open and bounded domain \(\mathcal{D} \in \mathbb{R}^d\). We assume \(\vect{x} \in \mathcal{D}\) represents spatial position, but \(\mathcal{D}\) can be viewed more generally. In practice, only discrete observations \(z_i = z(\vect{x}_i)\), \(i = 1,\dots,n\), of the function \(z(\vect{x})\) are available, which we stack into the column vector\footnote{We use boldface lowercase letters for vectors, boldface upper case letters for matrices, and normal face letters for scalars, including the components of vectors and matrices, such as \(z_i\) and \(K_{ij}\) respectively for the components of the vector \(\vect{z}\) and the matrix \(\tens{K}\). Also, \((\cdot)^\intercal\) denotes the transpose.} \(\vect{z} = [z_1,\dots,z_n]^{\intercal}\). Moreover, we assume each observation, \(z_i\), contains measurement error. 

A standard model for a non-stationary random process \(z\) is to utilize intrinsic random functions (IRFs) \citep{MATHERON-1973} by
\begin{equation}
    z(\vect{x}) = \mu(\vect{x}) + \delta(\vect{x}) + \epsilon(\vect{x}), \label{eq:model}
\end{equation}
where the deterministic mean function \(\mu(\vect{x})\) is the model trend (or drift). The residual error of the model, \(\delta(\vect{x})\), is a zero-mean intrinsically stationary stochastic process and is characterized by the spatial correlation of data. In contrast to \(\delta(\vect{x})\), the stochastic function \(\epsilon(\vect{x})\) represents the uncertainty of each observation at a given point. Note, \(\epsilon\) has a discrete realization, whereas \(\delta\) has a continuous sample path.  

The three terms \(\mu\), \(\delta\), and \(\epsilon\) in \eqref{eq:model} require further explanation:
\begin{enumerate}[wide]
    \item \emph{Function \(\mu\):} IRFs utilize a parametric linear model for the mean \(\mu\) that consists of a linear combination of admissible deterministic trend functions \(\mu(\vect{x}) = \vect{\basis}^{\intercal}(\vect{x}) \vect{\beta}\) where \(\vect{\basis} = [\basis_1,\dots,\basis_m]^{\intercal}: \mathcal{D} \to \mathbb{R}^{m}\) is a vector of \(m\) prescribed basis functions with \(m < n\). Often the basis functions are chosen to map the input data to a desirable low-dimensional feature space. Preferred basis functions are translation-invariant and orthogonal, which includes exponential, trigonometric, and polynomial functions (see \eg \citet[p. 308]{WACKERNAGEL-2013}). The basis functions at a given location \(\vect{x}_i\) can be used to form a design matrix \(\tens{\X} \in \mathbb{R}^{n \times m}\) with components \(\X_{ij} \coloneqq \basis_j(\vect{x}_i)\). We assume \(\tens{\X}\) has full column rank, \ie \(\operatorname{rank}(\tens{\X}) = m\). The vector \(\vect{\mu} = [\mu_1,\dots,\mu_n]^{\intercal}\) with \(\mu_i = \mu(\vect{x}_i)\) is then given by \(\vect{\mu} = \tens{\X} \vect{\beta}\). The parameter \(\vect{\beta} \in \mathbb{R}^m\) is the unknown column vector of regression coefficients of the basis functions to be estimated from the data.

    \item \emph{Function \(\delta\):}  The function \(\delta\) can be considered the residual error between the underlying process \(z\) and the prescribed mean function \(\mu\). Note that \(\delta\) is second-order stationary. Here, we assume a zero-mean Gaussian process prior for the stochastic function \(\delta\). At the discrete locations \(\vect{x}_i\), \(i=1,\ldots,n\), the random process \(\delta\) is represented by the vector \(\vect{\delta} = [\delta_1,\dots,\delta_n]^{\intercal}\) where \(\delta_i = \delta(\vect{x}_i)\). The covariance of the random vector \(\vect{\delta}\) can be written as \(\sigma^2 \tens{K}\), where \(\sigma^2\) is the variance and is constant over the domain since the process \(\delta\) is second-order stationary. The symmetric and positive-definite matrix \(\tens{K} \in \mathbb{R}^{n \times n}\) is the spatial correlation with components \(K_{ij} = K(\vect{x}_i,\vect{x}_j|\vect{\theta})\), where \(K\) is a prescribed kernel function satisfying Mercer's condition, and may generally depend on some hyperparameters \(\vect{\theta} = (\theta_1,\dots,\theta_p)\) (see \S \ref{sec:example}, in particular \eqref{eq:exp-decay} and \eqref{eq:matern-corr}, for suitable kernel functions). The variance \(\sigma^2\) is generally not known a priori and has to be inferred from the data.

    \item \emph{Function \(\epsilon\):} The stochastic function \(\epsilon\) represents the measurement uncertainty at a given point. We assume the noise has the same variance at all points (homoscedasticity) and moreover is Gaussian, \ie \(\epsilon(\vect{x}) \sim \mathcal{N}(0,\sigma_0^2)\), where \(\sigma_0^2\) may not be known a priori. At the discrete locations \(\vect{x}_i\), the random vector \(\vect{\epsilon} = [\epsilon_1,\dots,\epsilon_n]^{\intercal}\) with \(\epsilon_i = \epsilon(\vect{x}_i)\) has the covariance matrix \( \sigma_0^2 \tens{I}\) where \(\tens{I}\) here is the \(n \times n\) identity matrix.
\end{enumerate}

\begin{remark}
    Often, \(\sigma_0^2\) is referred to as the nugget as introduced by \citet{MATHERON-1962} in spatial statistics (see also \citet[Ch. 2, \S 2.3]{WACKERNAGEL-2013} or \cite[p. 130]{CRESSIE-1993}). The nugget quantifies the micro-scale variability that cannot be distinguished from the data uncertainty. In this context, \(\sigma^2\) quantifies the macro-scale variability of the data. 
\end{remark}

Overall, the semiparametric Gaussian process \eqref{eq:model} represented for the training data \(\vect{z} \sim \mathcal{N}(\vect{\mu},\tens{\Sigma})\) is given by
\begin{equation*}
    \vect{z} = \tens{\X} \vect{\beta} + \vect{\delta} + \vect{\epsilon},
\end{equation*}
with the total covariance of both random vectors \(\vect{\delta}\) and \(\vect{\epsilon}\) as
\begin{equation}
    \tens{\Sigma}_{\sigma^2,\sigma_0^2} = \sigma^2 \tens{K} + \sigma_0^2 \tens{I}, \label{eq:Sigma}
\end{equation}
which is the matrix form of the covariance function \eqref{eq:covariance-function} on training data points. The error variance \(\sigma^2\) and noise variance \(\sigma_0^2\) are treated as hyperparameters and indicated by subscripts on \(\tens{\Sigma}\). Unless needed, we often omit subscript notation for brevity. Note, when the residual of the model is assumed to be spatially uncorrelated, \ie \(\tens{K} = \tens{I}\), the variance of the error and noise are indistinguishable.


\subsection{Parameter Estimation, Preliminaries} \label{sec:prelim}

We review the generalized least squares and marginal maximum likelihood (MML) estimation of the parameter \(\vect{\beta}\) and hyperparameters \((\sigma^2,\sigma_0^2)\) assuming either of \(\sigma^2\), or \(\sigma_0^2\), or both are known. We estimate a fully unknown set of hyperparameters in \S \ref{sec:complete-estimate}. Also, we assume \((\sigma^2,\sigma_0^2)\) are the only covariance hyperparameters as estimating covariance with more complex parametrization is not the focus of this paper. However, in \S \ref{sec:optimal-parameters}, we provide a numerical example with additional covariance hyperparameters \(\vect{\theta}\) that enable flexibility regarding the choice of the correlation kernel function. 


\subsubsection{Generalized Least Squares Estimation of \texorpdfstring{\(\vect{\beta}\)}{beta}}

The conditional probability density function \(p(\vect{z}|\vect{\beta},\sigma^2,\sigma_0^2)\) of the data \(\vect{z} \sim \mathcal{N}(\tens{\X} \vect{\beta},\tens{\Sigma})\), given the parameters \((\vect{\beta},\sigma^2,\sigma_0^2)\), is the likelihood function given by the multivariate normal distribution,
\begin{equation}
    p(\vect{z}|\vect{\beta},\sigma^2,\sigma_0^2) = \frac{1}{\sqrt{(2 \pi)^n}} | \tens{\Sigma} |^{-\frac{1}{2}} \exp \left( -\frac{1}{2} (\vect{z} - \tens{\X} \vect{\beta})^{\intercal} \tens{\Sigma}^{-1} (\vect{z} - \tens{\X}\vect{\beta}) \right), \label{eq:likelihood}
\end{equation}
where \(|\tens{\Sigma}| > 0\) is the determinant of \(\tens{\Sigma}\).


In a simplified case, when \(\sigma^2\) and \(\sigma_0^2\) are known, the likelihood \eqref{eq:likelihood} is maximized by solving \( \partial(\log p(\vect{z}|\vect{\beta},\sigma^2,\sigma_0^2)) / \partial \vect{\beta} = \vect{0}\) for \(\vect{\beta}\) to obtain
\begin{equation}
    \hat{\vect{\beta}} = \left( \tens{\X}^{\intercal} \tens{\Sigma}^{-1} \tens{\X} \right)^{-1} \tens{\X}^{\intercal} \tens{\Sigma}^{-1} \vect{z}. \label{eq:beta-gls}
\end{equation}
We use the hat symbol to denote an estimation of a quantity. The solution \(\hat{\vect{\beta}}\) in the above is known by the Aitken's theorem for the generalized least squares estimation of \(\vect{\beta}\) (see \eg \citep[Ch. 13, \S 5]{MAGNUS-2019} and \citep[\S 2.3]{KARIYA-2004}). In a simplified case when the variance \(\sigma^2\) is assumed to be zero, \ie \(\tens{\Sigma} = \sigma_0^2 \tens{I}\), the solution \eqref{eq:beta-gls} reduces to the ordinary least squares estimation of \(\vect{\beta}\) for uncorrelated errors based on the assumptions of the Gauss-Markov theorem (see \eg \citep[Ch 13, \S 3]{MAGNUS-2019}).


\subsubsection{Marginal Likelihood and Estimation of \texorpdfstring{\(\sigma^2\)}{Error Variance} or \texorpdfstring{\(\sigma_0^2\)}{Noise Variances}} \label{sec:MLE-OR}

To estimate \(\sigma^2\) (when \(\sigma_0^2\) is known) or \(\sigma_0^2\) (when \(\sigma^2\) is known), we maximize the marginal likelihood function where \(\vect{\beta}\) is marginalized out. It is straightforward to derive the marginal likelihood of a Gaussian process (see \eg \citet[p. 247]{MCCULLAGH-1989} or \citet[p. 287]{SEBER-2012}). However, it is beneficial to re-derive the marginal likelihood via a projection operator that will be used in our later development.

We rewrite the argument inside the exponential function in \eqref{eq:likelihood} as,
\begin{equation*}
    (\vect{z} - \tens{\X} \vect{\beta})^{\intercal} \tens{\Sigma}^{-1}(\vect{z} - \tens{\X} \vect{\beta}) = \| \vect{z} - \tens{\X} \vect{\beta} \|^2_{\tens{\Sigma}^{-1}},
\end{equation*}
which is the Mahalanobis norm of the mean deviation \(\vect{z} - \vect{\mu}\) with respect to the metric tensor \(\tens{\Sigma}^{-1}\). Define
\begin{equation}
    \tens{P} \coloneqq \tens{I} - \tens{\X} \left( \tens{\X}^{\intercal} \tens{\Sigma}^{-1} \tens{\X} \right)^{-1} \tens{\X}^{\intercal} \tens{\Sigma}^{-1}, \label{eq:P}
\end{equation}
which is an orthogonal projection matrix with respect to the inner product induced by \(\tens{\Sigma}^{-1}\). The right null space (kernel) of \(\tens{P}\) is spanned by the columns of \(\tens{\X}\) since \(\tens{P} \tens{\X} = \tens{0}\). The left null space (cokernel) of \(\tens{P}\) is spanned by the columns of \(\tens{\X}^{\intercal} \tens{\Sigma}^{-1}\) since \(\tens{\X}^{\intercal} \tens{\Sigma}^{-1} \tens{P} = \tens{0}\). Additionally, define the symmetric matrix,
\begin{equation}
    \tens{M} \coloneqq \tens{\Sigma}^{-1} \tens{P}, \label{eq:M}
\end{equation}
where \(\operatorname{ker}(\tens{M}) = \operatorname{span}(\tens{\X})\) and \(\operatorname{coker}(\tens{M}) = \operatorname{span}(\tens{\X}^{\intercal})\).

\begin{claim}[Orthogonal Decomposition of the Mean Deviation] \label[claim]{claim:decomposition}
    It holds that
    \begin{equation}
        \| \vect{z} - \tens{\X} \vect{\beta} \|^2_{\tens{\Sigma}^{-1}} = \| \vect{z} \|^2_{\tens{M}} + \| \vect{\beta} - \hat{\vect{\beta}} \|^2_{\tens{\X}^{\intercal} \tens{\Sigma}^{-1} \tens{\X}}, \label{eq:pythagorean}
    \end{equation}
    where \(\hat{\vect{\beta}}\) is the generalized least squares estimation of \(\vect{\beta}\) given in \eqref{eq:beta-gls}.
\end{claim}

\begin{proof}
    Since \(\tens{I} - \tens{P}\) is the complement projection operator of \(\tens{P}\), we can write the orthogonality relation \(\tens{P}(\vect{z} - \tens{\X} \vect{\beta}) \perp_{\tens{\Sigma}^{-1}} (\tens{I} - \tens{P}) (\vect{z} - \tens{\X}\vect{\beta}) \) where \(\perp_{\tens{\Sigma}^{-1}}\) denotes orthogonal with respect to the inner product induced by the metric tensor \(\tens{\Sigma}^{-1}\). As a result, the Pythagorean relationship,
    \begin{equation}
        \| \vect{z} - \tens{\X} \vect{\beta} \|^2_{\tens{\Sigma}^{-1}} = \| \tens{P} (\vect{z} - \tens{\X} \vect{\beta}) \|^2_{\tens{\Sigma}^{-1}} + \| (\tens{I} - \tens{P}) (\vect{z} - \tens{\X} \vect{\beta} ) \|^2_{\tens{\Sigma}^{-1}}, \label{eq:pythagorean-2}
    \end{equation}
    holds. Recall that \(\tens{P} \tens{\X} = \tens{0}\). Thus, \( \| \tens{P}(\vect{z} - \tens{\X} \vect{\beta}) \|^2_{\tens{\Sigma}^{-1}} =  \| \tens{P} \vect{z} \|^2_{\tens{\Sigma}^{-1}} = \vect{z}^{\intercal} (\tens{P}^{\intercal} \tens{\Sigma}^{-1} \tens{P}) \vect{z}\). Realize from \eqref{eq:P} that \(\tens{P}^{\intercal} = \tens{\Sigma}^{-1} \tens{P} \tens{\Sigma}\), so, \(\tens{P}^{\intercal} \tens{\Sigma}^{-1}\tens{P} = \tens{\Sigma}^{-1}\tens{P}^2\). But, a projection operator is idempotent, \ie \(\tens{P}^2 = \tens{P}\) (see \eg \citep[\S 12.3]{GRAYBILL-1983}), simplifying \(\tens{P}^{\intercal} \tens{\Sigma}^{-1} \tens{P} = \tens{\Sigma}^{-1} \tens{P} = \tens{M}\), and the first term on the right side of \eqref{eq:pythagorean-2} becomes \( \| \tens{P}(\vect{z} - \tens{\X} \vect{\beta}) \|^2_{\tens{\Sigma}^{-1}} = \| \vect{z} \|^2_{\tens{M}}\). As for the second right term of \eqref{eq:pythagorean-2}, we have \( (\tens{I} - \tens{P})(\vect{z} - \tens{\X} \vect{\beta}) = (\tens{I} - \tens{P}) \vect{z} - \tens{\X} \vect{\beta} \). From \eqref{eq:beta-gls} and \eqref{eq:P} we have \((\tens{I} - \tens{P})\vect{z} = \tens{\X} \hat{\vect{\beta}}\). Thus, the second right term of \eqref{eq:pythagorean-2} becomes \(\| -\tens{\X}(\vect{\beta} - \hat{\vect{\beta}}) \|^2_{\tens{\Sigma}^{-1}}\) which yields the second right term of \eqref{eq:pythagorean}.
\end{proof}

By the orthogonal decomposition in \Cref{claim:decomposition}, the likelihood function \eqref{eq:likelihood} becomes
\begin{equation*}
    p(\vect{z}|\vect{\beta},\sigma^2,\sigma_0^2) = \frac{1}{\sqrt{(2 \pi)^n}} |\tens{\Sigma}|^{-\frac{1}{2}} \exp \left( -\frac{1}{2} \| \vect{z} \|^2_{\tens{M}} \right) \exp \left( -\frac{1}{2} \| \vect{\beta} - \hat{\vect{\beta}} \|^2_{\tens{\X}^{\intercal} \tens{\Sigma}^{-1} \tens{\X}} \right).
\end{equation*}
Integrating the above (see \eg \cite[Theorem 10.6.1]{GRAYBILL-1983}) in the domain \(\vect{\beta} \in \mathbb{R}^{m} \) yields the marginal likelihood,
\begin{equation}
    p(\vect{z}|\sigma^2,\sigma_0^2) = \frac{1}{\sqrt{(2 \pi)^{n-m}}} |\tens{\Sigma}|^{-\frac{1}{2}} |\tens{\X}^{\intercal} \tens{\Sigma}^{-1} \tens{\X}|^{-\frac{1}{2}} \exp \left( -\frac{1}{2} \| \vect{z} \|^2_{\tens{M}} \right). \label{eq:marginal-likelihood}
\end{equation}

Finding the maximum of the above likelihood function when either of \(\sigma^2\) or \(\sigma_0^2\) is zero is straightforward. Namely, when \(\sigma_0^2 = 0\), and hence, \(\tens{\Sigma} = \sigma^2 \tens{K}\) and \(\tens{P} = \tens{I} - \tens{\X} (\tens{\X}^{\intercal} \tens{K}^{-1} \tens{\X})^{-1} \tens{\X}^{\intercal} \tens{K}^{-1}\), it can be shown (see \eg \cite[Theorem 3.3]{SEBER-2012}) that the above marginal likelihood attains its maximum at
\begin{equation}
    \hat{\sigma}^2 = \frac{1}{n-m} \| \vect{z} - \tens{\X} \hat{\vect{\beta}} \|^2_{\tens{K}^{-1}} = \frac{1}{n-m} \| \vect{z} \|^2_{\tens{K}^{-1} \tens{P}}.
    \label{eq:variance-trivial}
\end{equation}
The right equality in the above can be verified by setting \(\vect{\beta} = \hat{\vect{\beta}}\) in \Cref{claim:decomposition}. Conversely, if \(\sigma^2 = 0\), and hence \(\tens{\Sigma} = \sigma_0^2 \tens{I}\) and \(\tens{P} = \tens{I} - \tens{\X} (\tens{\X}^{\intercal} \tens{\X})^{-1} \tens{\X}^{\intercal}\), an estimate is given by
\begin{equation}
    \hat{\sigma}_0^2 = \frac{1}{n-m} \| \vect{z} - \tens{\X} \hat{\vect{\beta}} \|^2 = \frac{1}{n-m} \| \vect{z} \|^2_{\tens{P}}.
    \label{eq:nugget-trivial}
\end{equation}
The solutions \eqref{eq:variance-trivial} and \eqref{eq:nugget-trivial} can also be derived by the restricted maximum likelihood method (ReML) of \cite{PATTERSON-1971} using the distribution of contrasts (see also \citet[p. 170]{STEIN-1999}). ReML takes into account the loss of degrees of freedom by the model (here, \(m\)), which eliminates the bias of the solution. Because of this, the solutions \eqref{eq:variance-trivial} and \eqref{eq:nugget-trivial} are unbiased. We also note that if we had started with maximizing the likelihood function (see \cite[Equation 14]{HARTLEY-1967}) instead of the \emph{marginal} likelihood function, the solutions \eqref{eq:variance-trivial} and \eqref{eq:nugget-trivial} would become biased.


\section{MML Estimation of Error and Noise Variances}  \label{sec:complete-estimate}

When \(\sigma \sigma_0 \neq 0\), an analytical solution for \((\hat{\sigma},\hat{\sigma}_0)\) with maximum marginal likelihood estimation \eqref{eq:marginal-likelihood} is not known. Usually, a nonlinear numerical optimization is performed to find the hyperparameters. In this section, we derive an analytical approach to reduce the space of the hyperparameters so that the hyperparameter optimization reduces to a much simpler univariate root-finding problem.


\subsection{Derivatives of the Marginal Likelihood Function} \label{sec:derivatives}

To find optimal hyperparameters, we maximize the logarithm of the marginal likelihood function \eqref{eq:marginal-likelihood}. Define
\begin{equation*}
    \ell \coloneqq \log p(\vect{z}|\sigma^2,\sigma_0^2),
\end{equation*}
which is,
\begin{equation}
    \ell = -\frac{(n-m)}{2} \log (2 \pi) - \frac{1}{2} \log |\tens{\Sigma}| - \frac{1}{2} \log |\tens{\X}^{\intercal} \tens{\Sigma}^{-1} \tens{\X}| - \frac{1}{2} \| \vect{z} \|^2_{\tens{M}}. \label{eq:log-marginal-likelihood}
\end{equation}
To find the maximum of \(\ell\) with respect to a hyperparameter, in \Cref{prop:der-L} below, we derive analytic expressions for the first and higher-order derivatives of \(\ell\). However, in practice, only up to the second-order derivative is needed to maximize \(\ell\).  For generality, we consider derivatives with respect to an arbitrary hyperparameter, called \(\theta\). Here, \(\theta\) will usually denote either \(\sigma^2\) or \(\sigma_0^2\), or their combination as presented in \S \ref{sec:estimate}. However, \(\theta\) can represent any desired hyperparameters. 

\begin{lemma}[Derivative of \(\tens{M}\)] \label[lemma]{lem:der-M}
    Let \(\theta\) denote a hyperparameter of \(\tens{\Sigma}\) and accordingly \(\tens{M}\). Then,
    \begin{subequations}
    \begin{equation}
        \dot{\tens{M}} = - \tens{M} \dot{\tens{\Sigma}} \tens{M}, \label{eq:M-dot}
    \end{equation}
    where \(\dot{\tens{M}} = \partial \tens{M} / \partial \theta\) and \(\dot{\tens{\Sigma}} = \partial \tens{\Sigma} / \partial \theta\). Furthermore, if \(\tens{\Sigma}\) is a linear function of the hyperparameter \(\theta\) (such as for \(\sigma^2\) and \(\sigma_0^2\) in \eqref{eq:Sigma}), then the \(k\)\textsuperscript{th} order derivative of \(\tens{M}\) is,
    \begin{equation}
        \frac{\partial^k \tens{M}}{\partial \theta^k} = (-1)^k k! \tens{M} (\dot{\tens{\Sigma}} \tens{M})^k. \label{eq:der-k-M}
    \end{equation}
    \end{subequations}
\end{lemma}

\begin{proof}
    From \eqref{eq:P} and \eqref{eq:M} we have \(\tens{M} = \tens{\Sigma}^{-1} - \tens{\Sigma}^{-1} \tens{\X} (\tens{\X}^{\intercal} \tens{\Sigma}^{-1} \tens{\X})^{-1} \tens{\X}^{\intercal} \tens{\Sigma}^{-1}\). To take the derivative of the inverse of a matrix, such as \(\tens{\Sigma}^{-1}\) or \((\tens{\X}^{\intercal} \tens{\Sigma}^{-1} \tens{\X})^{-1}\), we use the identity \(\partial (\tens{A}^{-1}) / \partial \theta = - \tens{A}^{-1} (\partial \tens{A} / \partial \theta) \tens{A}^{-1}\), which can be shown by taking the derivative of \(\tens{A} \tens{A}^{-1}  = \tens{I}\). Thus,
    \begin{align*}
        \dot{\tens{M}} = &-\tens{\Sigma}^{-1} \dot{\tens{\Sigma}} \tens{\Sigma}^{-1} + (\tens{\Sigma}^{-1} \dot{\tens{\Sigma}} \tens{\Sigma}^{-1}) \tens{\X} (\tens{\X}^{\intercal} \tens{\Sigma}^{-1} \tens{\X})^{-1} \tens{\X}^{\intercal} \tens{\Sigma}^{-1} \\
        & - \tens{\Sigma}^{-1} \tens{\X} (\tens{\X}^{\intercal} \tens{\Sigma}^{-1} \tens{\X})^{-1} \tens{\X}^{\intercal} (\tens{\Sigma}^{-1} \dot{\tens{\Sigma}} \tens{\Sigma}^{-1}) \tens{\X} (\tens{\X}^{\intercal} \tens{\Sigma}^{-1} \tens{\X})^{-1} \tens{\X}^{\intercal} \tens{\Sigma}^{-1} \\
        & + \tens{\Sigma}^{-1} \tens{\X} (\tens{\X}^{\intercal} \tens{\Sigma}^{-1} \tens{\X})^{-1} \tens{\X}^{\intercal} (\tens{\Sigma}^{-1} \dot{\tens{\Sigma}} \tens{\Sigma}^{-1}).
    \end{align*}
    The above terms can be factored into the product,
    \begin{equation*}
        \dot{\tens{M}} = - \underbrace{\left( \tens{\Sigma}^{-1} - \tens{\Sigma}^{-1} \tens{\X} (\tens{\X}^{\intercal} \tens{\Sigma}^{-1} \tens{\X})^{-1} \tens{\X}^{\intercal} \tens{\Sigma}^{-1} \right)}_{\tens{M}} \dot{\tens{\Sigma}} \underbrace{\left( \tens{\Sigma}^{-1} - \tens{\Sigma}^{-1} \tens{\X} (\tens{\X}^{\intercal} \tens{\Sigma}^{-1} \tens{\X})^{-1} \tens{\X}^{\intercal} \tens{\Sigma}^{-1} \right)}_{\tens{M}},
    \end{equation*}
    which concludes \eqref{eq:M-dot}. The relation \eqref{eq:der-k-M} can be shown by induction and using the fact that if \(\tens{\Sigma}\) is linear in \(\theta\), the second derivative, \(\ddot{\tens{\Sigma}}\), vanishes.
\end{proof}

\begin{proposition}[Derivatives of Log Marginal Likelihood Function] \label[proposition]{prop:der-L}
    Let \(\theta\) denote a hyperparameter of the matrix \(\tens{\Sigma}\) and accordingly the log marginal likelihood function \(\ell\) in \eqref{eq:log-marginal-likelihood}. Then,
    \begin{subequations}
    \begin{equation}
        \frac{\partial \ell}{\partial \theta} = -\frac{1}{2} \trace(\dot{\tens{\Sigma}} \tens{M}) + \frac{1}{2} \vect{z}^{\intercal} \left( \tens{M} \dot{\tens{\Sigma}} \tens{M} \right) \vect{z}, \label{eq:der-L}
    \end{equation}
    where \(\dot{\tens{\Sigma}} = \partial \tens{\Sigma} / \partial \theta\). Furthermore, if \(\tens{\Sigma}\) is a linear function of the hyperparameter \(\theta\), then, the \(k\)\textsuperscript{th} order derivative of \(\ell\) is,
    \begin{equation}
        \frac{\partial^k \ell}{\partial \theta^k} = \frac{1}{2}(-1)^k (k-1)! \left( \trace\Big( (\dot{\tens{\Sigma}} \tens{M})^k \Big) - k\vect{z}^{\intercal} \Big( \tens{M} (\dot{\tens{\Sigma}} \tens{M})^k \Big) \vect{z} \right).
        \label{eq:der-k-L}
    \end{equation}
    \end{subequations}
\end{proposition}

\begin{proof}
    To take the derivative of the determinants in \eqref{eq:log-marginal-likelihood}, we use the Jacobi formula, \ie \(\partial |\tens{A}| / \partial \theta = |\tens{A}| \trace(\tens{A}^{-1} \partial \tens{A}/\partial \theta)\). Thus, the derivative of \eqref{eq:log-marginal-likelihood} is,
    \begin{equation*}
        \frac{\partial \ell}{\partial \theta} = -\frac{1}{2} \trace(\tens{\Sigma}^{-1} \dot{\tens{\Sigma}}) + \frac{1}{2} \trace\left( (\tens{\X}^{\intercal} \tens{\Sigma}^{-1} \tens{\X})^{-1} \tens{\X}^{\intercal} (\tens{\Sigma}^{-1} \dot{\tens{\Sigma}} \tens{\Sigma}^{-1}) \tens{\X} \right) + \frac{1}{2} \vect{z}^{\intercal} \left(\tens{M} \dot{\tens{\Sigma}} \tens{M} \right) \vect{z}.
    \end{equation*}
    For the second term on the right, we used \(\partial (\tens{\Sigma}^{-1}) / \partial \theta = -\tens{\Sigma}^{-1} \dot{\tens{\Sigma}} \tens{\Sigma}^{-1}\). For the third term on the right, \Cref{lem:der-M} for \(\dot{\tens{M}}\) was applied. For both the first and second terms on the right, we can use the cyclic property of trace to write \(\trace(\tens{\Sigma}^{-1} \dot{\tens{\Sigma}}) = \trace(\dot{\tens{\Sigma}} \tens{\Sigma}^{-1})\) and
    \begin{equation*}
        \trace\left( (\tens{\X}^{\intercal} \tens{\Sigma}^{-1} \tens{\X})^{-1} \tens{\X}^{\intercal} (\tens{\Sigma}^{-1} \dot{\tens{\Sigma}} \tens{\Sigma}^{-1}) \tens{\X} \right) = \trace \left( \dot{\tens{\Sigma}} \tens{\Sigma}^{-1} \tens{\X} (\tens{\X}^{\intercal} \tens{\Sigma}^{-1} \tens{\X})^{-1} \tens{\X}^{\intercal} \tens{\Sigma}^{-1} \right).
    \end{equation*}
    We then can arrive at
    \begin{equation*}
        \frac{\partial \ell}{\partial \theta} = -\frac{1}{2} \trace\bigg( \dot{\tens{\Sigma}} \underbrace{(\tens{\Sigma}^{-1} - \tens{\Sigma}^{-1} \tens{\X}(\tens{\X}^{\intercal} \tens{\Sigma}^{-1} \tens{\X})^{-1} \tens{\X}^{\intercal} \tens{\Sigma}^{-1})}_{\tens{M}} \bigg) + \frac{1}{2} \vect{z}^{\intercal} \left( \tens{M} \dot{\tens{\Sigma}} \tens{M} \right) \vect{z},
    \end{equation*}
    which concludes \eqref{eq:der-L}. Similarly, \eqref{eq:der-k-L} can be verified by induction, applying \Cref{lem:der-M} and \(\partial^2 \tens{\Sigma} / \partial \theta^2 = \tens{0}\).
\end{proof}

\begin{remark}[Absence of Basis Functions]
By assuming the trivial mean function \ie \(\mu = 0\), we have \(\tens{\X} = \tens{0}\) so \(\tens{P} = \tens{I}\), and the matrix \(\tens{M}\) simplifies to the precision matrix \(\tens{\Sigma}^{-1}\). As such, the derivative of \(\ell\) in \Cref{prop:der-L} is greatly simplified to a commonly known form in the literature (see \eg \citet[\S 6.1]{MACKAY-1998}, \citep[Equation 5.9]{RASMUSSEN-2006}, or \citep[Equation 15.23]{MURPHY-2012}).
\end{remark}

An important virtue of \Cref{lem:der-M} and \Cref{prop:der-L} is that the derivatives of \(\tens{M}\) and \(\ell\) are tractably represented by \(\tens{M}\) (as opposed to the lengthy expressions by only \(\tens{\Sigma}\) and \(\tens{\X}\)). This is an important result as it facilitates our forthcoming developments.


\subsection{Estimation of Error and Noise Variances} \label{sec:estimate}

Here we aim to estimate the set of hyperparameters \((\sigma^2,\sigma_0^2)\) assuming \(\sigma \sigma_0 \neq 0\). 
We show that this can be achieved by reformulating the problem to estimate the hyperparameters \((\sigma^2,\sigma_0^2/\sigma^2)\) instead. Define the ratio between the noise and error variances by,
\begin{equation}
    \eta \coloneqq \frac{\sigma_0^2}{\sigma^2}.
\end{equation}
Then, the covariance matrix \eqref{eq:Sigma} is represented by the new hyperparameters as
\begin{equation}
    \tens{\Sigma}_{\sigma^2,\eta} = \sigma^2 \tens{K}_{\eta}, \qquad \text{where} \qquad \tens{K}_{\eta} \coloneqq \tens{K} + \eta \tens{I}. \label{eq:Sigma-K}
\end{equation}
Accordingly, we inherit the same sub-index notation of hyperparameters for the matrices \(\tens{M}_{\sigma^2,\eta}\) and \(\tens{P}_{\sigma^2,\eta}\), since both of these matrices depend on \(\tens{\Sigma}_{\sigma^2,\eta}\). Note that \(\tens{P}_{\sigma^2,\eta}\) is independent of the error variance as \(\sigma^2\) cancels out in its formulation, so, we only need to write \(\tens{P}_{\eta}\).

We proceed as follows. First, in \Cref{thm:sigma}, we assume \(\eta\) is given, and we set \(\theta = \sigma^2\) to find the global maximum of \(\ell(\sigma^2,\eta)\), which we denote \(\hat{\sigma}^2\). Such a solution to the error variance depends on \(\eta\), which we indicate by \(\hat{\sigma}^2(\eta)\). Second, in \Cref{thm:eta}, we set \(\theta = \eta\) to find the local maximum of the \emph{profiled} function \(\ell(\hat{\sigma}^2(\eta),\eta)\), which we denote \(\hat{\eta}\). Once \(\hat{\eta}\) and \(\hat{\sigma}^2(\hat{\eta})\) are known, an estimate of the noise variance is obtained by \(\hat{\sigma}_0^2 = \hat{\eta} \hat{\sigma}^2(\hat{\eta})\).

\begin{remark} \label[remark]{rem:z-rangeX}
    In the following theorems, we assume \(\vect{z} \notin \range(\tens{\X})\), \ie the observed data, \(\vect{z}\), is not a linear combination of the basis functions \(\vect{\basis}_i\). If \(\vect{z} \in \range(\tens{\X})\), since \(\tens{\X}\) is the kernel of the projection matrix \(\tens{P}_{\eta}\), we have \(\tens{P}_{\eta} \vect{z} = \vect{0}\), which yields \(\tens{M}_{\sigma^2,\eta} \vect{z} = \tens{0}\) and \(\| \vect{z} \|_{\tens{M}} = 0\). In such a case, \(\ell\) is independent of \(\vect{z}\) and becomes unbounded with a logarithmic singularity at \(\sigma = 0\) (see \eqref{eq:log-marginal-likelihood}), which  leads to the trivial solution \(\hat{\sigma} = 0\), as expected. Thus, for non-trivial problems, we henceforth assume \(\vect{z} \notin \range(\tens{\X})\).
\end{remark}

\begin{theorem}[Estimation of Variance] \label[theorem]{thm:sigma}
    Suppose \(\vect{z} \notin \range(\tens{\X})\). For a given \(\eta\), the log marginal likelihood \(\ell_{\eta}(\sigma^2) \coloneqq \ell(\sigma^2,\eta)\) has a strict global maximum at
    \begin{subequations}
    \begin{equation}
        \hat{\sigma}^2(\eta) = \frac{1}{n-m} \| \vect{z} \|^2_{\tens{M}_{1,\eta}}, \label{eq:sigma-hat}
    \end{equation}
    where
    \begin{equation}
        \tens{M}_{1,\eta} = \tens{K}_{\eta}^{-1} - \tens{K}_{\eta}^{-1} \tens{\X} (\tens{\X}^{\intercal} \tens{K}_{\eta}^{-1} \tens{\X})^{-1} \tens{\X}^{\intercal} \tens{K}_{\eta}^{-1}. \label{eq:M1}
    \end{equation}
    \end{subequations}
\end{theorem}

\begin{proof}
    Set \(\theta = \sigma^2\) in \Cref{prop:der-L}, so, here, the derivatives are with respect to \(\sigma^2\). We have, \(\dot{\tens{\Sigma}}_{\sigma^2,\eta} = \tens{K}_{\eta}\). Recall from \eqref{eq:M} that \(\tens{M} = \tens{\Sigma}^{-1} \tens{P}\) and write \(\tens{M}_{\sigma^2,\eta} = \sigma^{-2} \tens{K}_{\eta}^{-1} \tens{P}_{\eta}\). We compute each of the terms in the \(k\)\textsuperscript{th} derivative of \(\ell\) given in \eqref{eq:der-k-L}. 
    
    First,
    \begin{align}
        \trace\Big( (\dot{\tens{\Sigma}} \tens{M})^k \Big) &= \sigma^{-2k} \trace(\tens{P}_\eta^k) \notag \\
        &= \sigma^{-2k} \trace(\tens{P}_\eta) \notag \\
        &= \sigma^{-2k} \trace(\tens{I}_{n \times n} - \tens{\X}(\tens{\X}^{\intercal} \tens{K}_{\eta} \tens{\X})^{-1} \tens{\X}^{\intercal} \tens{K}_{\eta}^{-1}) \notag \\
        &= \sigma^{-2k} \left( \trace(\tens{I}_{n \times n}) - \trace((\tens{\X}^{\intercal} \tens{K}_\eta^{-1} \tens{\X})^{-1} (\tens{\X}^{\intercal} \tens{K}_{\eta}^{-1} \tens{\X})) \right) \notag \\
        &= \sigma^{-2k} \left( \trace(\tens{I}_{n \times n}) - \trace(\tens{I}_{m \times m}) \right) \notag \\
        &= \sigma^{-2k} (n - m). \label{eq:dsm}
    \end{align}
    On the second line in the above, we used the idempotent property of the projection matrix \(\tens{P}_{\eta}\), that is \(\tens{P}_{\eta}^k = \tens{P}_{\eta}\). In the fourth line, we used the cyclic property of the trace operator. \(\tens{I}_{n \times n}\) and \(\tens{I}_{m \times m}\) are identity matrices of size \(n \times n\) and \(m \times m\), respectively.

    Second,
    \begin{align}
        \tens{M} (\dot{\tens{\Sigma}} \tens{M})^k &= (\sigma^{-2}\tens{K}_{\eta}^{-1} \tens{P}_{\eta}) ( \tens{K}_{\eta} \sigma^{-2} \tens{K}_{\eta}^{-1} \tens{P}_{\eta})^{k} \notag \\
        &= \sigma^{-2(1+k)} \tens{K}_{\eta}^{-1} \tens{P}_{\eta}^{k+1} \notag \\
        &= \sigma^{-2(1+k)}\tens{K}_{\eta}^{-1} \tens{P}_{\eta} \notag \\
        &= \sigma^{-2(1+k)} \tens{M}_{1,\eta}. \label{eq:msm}
    \end{align}
    On the third line in the above, we again used the idempotent property of the projection matrix \(\tens{P}_{\eta}\). Applying \eqref{eq:dsm} and \eqref{eq:msm} in the partial derivative of \(\ell\) in \eqref{eq:der-k-L} yields
    \begin{equation}
        \frac{\partial^k \ell_{\eta}(\sigma^2)}{\partial (\sigma^2)^k} = \frac{(-1)^k}{2 \sigma^{2k}} (k-1)! \left( (n-m) - k \sigma^{-2} \vect{z}^{\intercal} \tens{M}_{1,\eta} \vect{z}  \right). \label{eq:der2-L-sigma2}
    \end{equation}
    Since by the hypothesis \(\tens{M}_{1,\eta} \vect{z} \neq \vect{0}\), the above relation has a single root. When \(k = 1\), the root at \(\hat{\sigma}^2\) is given by \eqref{eq:sigma-hat}. Moreover since 
    \begin{equation*}
        \left. \frac{\partial^k \ell_{\eta}(\sigma^2)}{\partial (\sigma^2)^k} \right|_{\hat{\sigma}^2}  = -\frac{(-1)^k}{2 \hat{\sigma}^{2k}} (k-1) (k-1)! (n-m),
    \end{equation*}
    \(\ell_{\eta}(\hat{\sigma}^2)\) is a maximum because the second derivative \(k = 2\) is negative.
\end{proof}

\begin{remark}[Behavior at \(\eta = 0\) and \(\eta \to \infty\)]
    When \(\eta\) vanishes, and thus \(\hat{\sigma}_0^2 = 0\), the solution to \(\hat{\sigma}^2\) in \Cref{thm:sigma} falls back to the trivial solution known before in \eqref{eq:variance-trivial}. Conversely, when \(\eta \to \infty\), we have \(\tens{K}_{\eta} \to \eta \tens{I}\), \(\tens{P}_{\eta} \to \tens{I} - \tens{\X}(\tens{\X}^{\intercal} \tens{\X})^{-1} \tens{\X}^{-1}\), and \(\tens{M}_{1,\eta}\ \to \eta^{-1} \tens{P}\). By \Cref{thm:sigma}, \(\hat{\sigma} \to 0\), however, \(\hat{\sigma}_0^2 = \eta \hat{\sigma}^2 \to \| \vect{z} \|^2_{\tens{P}} / (n-m)\), which is known before by \eqref{eq:nugget-trivial}.
\end{remark}

We calculate the first and second derivatives of \(\ell\) with respect to \(\eta\) in \Cref{prop:der-eta}. With a slight alternative formulation, these derivatives are also represented in \Cref{thm:eta} to estimate \(\eta\).

\begin{proposition}[Derivatives of \(\ell\) with respect to \(\eta\)] \label[proposition]{prop:der-eta}
    Let \(\ell_{\hat{\sigma}^2(\eta)}(\eta) \coloneqq \ell(\hat{\sigma}^2(\eta),\eta)\) denote the profile log marginal likelihood that is locally maximized over \(\sigma^2\) at \(\hat{\sigma}^2(\eta)\) by \Cref{thm:sigma}. Then, the first and second total derivatives of \(\ell_{\hat{\sigma}^2(\eta)}(\eta)\) are
    \begin{subequations}
        \begin{equation}
            \frac{\mathrm{d} \ell_{\hat{\sigma}^2(\eta)}(\eta)}{\mathrm{d} \eta} = 
            -\frac{1}{2} \left( \trace(\tens{M}_{1,\eta}) - \frac{1}{\hat{\sigma}^{2}(\eta)} \vect{z}^{\intercal} \tens{M}_{1,\eta}^2 \vect{z} \right), \label{eq:der-ell-eta}
        \end{equation}
        and
        \begin{equation}
            \frac{\mathrm{d}^2 \ell_{\hat{\sigma}^2(\eta)}(\eta)}{\mathrm{d} \eta^2} = 
            \frac{1}{2} \left( \trace(\tens{M}_{1,\eta}^2) 
            -\frac{2}{\hat{\sigma}^2(\eta)} \vect{z}^{\intercal} \tens{M}_{1,\eta}^3 \vect{z}
            +\frac{1}{(n-m)(\hat{\sigma}^2(\eta))^2} \left( \vect{z}^{\intercal} \tens{M}_{1,\eta}^2 \vect{z} \right)^2
            \right),
            \label{eq:der2-ell-eta}
        \end{equation}
    \end{subequations}
    where \(\hat{\sigma}^2(\eta)\) is given by \eqref{eq:sigma-hat}.
\end{proposition}

\begin{proof}
    The total derivatives of \(\ell_{\hat{\sigma}^2(\eta)}(\eta)\) with respect to \(\eta\) are
    \begin{subequations}
    \begin{equation}
        \frac{\mathrm{d} \ell_{\hat{\sigma}^2(\eta)}(\eta)}{\mathrm{d} \eta} = \left. \frac{\partial \ell(\sigma^2,\eta)}{\partial \eta} \right|_{\hat{\sigma}^2(\eta)} + \frac{\mathrm{d} \hat{\sigma}^2(\eta)}{\mathrm{d} \eta} \left. \frac{\partial \ell(\sigma^2,\eta)}{\partial \sigma^2} \right|_{\hat{\sigma}^2(\eta)}, \label{eq:total-d-l}
    \end{equation}
    and
    \begin{align}
        \frac{\mathrm{d}^2 \ell_{\hat{\sigma}^2(\eta)}(\eta)}{\mathrm{d} \eta^2} =& \left. \frac{\partial^2 \ell(\sigma^2,\eta)}{\partial \eta^2} \right|_{\hat{\sigma}^2(\eta)} + \left. 2 \frac{\mathrm{d} \hat{\sigma}^2(\eta)}{\mathrm{d} \eta} \frac{\partial^2 \ell(\sigma^2,\eta)}{\partial \sigma^2 \partial \eta} \right|_{\hat{\sigma}^2(\eta)}
        + \left. \left( \frac{\mathrm{d} \hat{\sigma}^2(\eta)}{\mathrm{d} \eta} \right)^2 \frac{\partial^2 \ell(\sigma^2,\eta)}{\partial (\sigma^2)^2} \right|_{\hat{\sigma}^2(\eta)} \notag \\
        & + \left. \frac{\mathrm{d}^2 \hat{\sigma}^2(\eta)}{\mathrm{d} \eta^2} \frac{\partial \ell(\sigma^2,\eta)}{\partial \sigma^2} \right|_{\hat{\sigma}^2(\eta)}. \label{eq:total-dd-l}
    \end{align}
    \end{subequations}
    By the definition of \(\hat{\sigma}^2(\eta)\) in \Cref{thm:sigma}, \(\partial \ell(\sigma^2,\eta) / \partial \sigma^2 = 0\) at \(\sigma^2 = \hat{\sigma}^2(\eta)\), so the last term in \eqref{eq:total-d-l} and \eqref{eq:total-dd-l} vanishes. To find the partial derivatives with respect to \(\eta\) in the above, we apply \Cref{prop:der-L} by setting \(\theta = \eta\). We have \(\dot{\tens{\Sigma}}_{\hat{\sigma}^2(\eta),\eta} = \hat{\sigma}^2(\eta) \tens{I}\), where, here, dot denotes the derivative with respect to \(\eta\). Also, recall that \(\tens{M}_{\hat{\sigma}^2(\eta),\eta} = \hat{\sigma}^{-2}(\eta) \tens{M}_{1,\eta}\). Thus, the two terms in the partial derivatives of \(\ell\) in \eqref{eq:der-L} and \eqref{eq:der-k-L} are,
    \begin{subequations} \label{eq:terms-der-k-L-eta}
    \begin{equation}
        \left(\dot{\tens{\Sigma}}_{\hat{\sigma}^2(\eta),\eta} \tens{M}_{\hat{\sigma}^2(\eta),\eta}\right)^k = \tens{M}_{1,\eta}^k,
    \end{equation}
    and
    \begin{equation}
        \tens{M}_{\hat{\sigma}^2(\eta),\eta} \left(\dot{\tens{\Sigma}}_{\hat{\sigma}^2(\eta),\eta} \tens{M}_{\hat{\sigma}^2(\eta),\eta}\right)^k = \hat{\sigma}^{-2}(\eta) \tens{M}_{1,\eta}^{k+1}.
    \end{equation}
    \end{subequations}
    By applying the above terms at \(k = 1\) in \eqref{eq:der-L}, the first partial derivative of \(\ell\) becomes,
    \begin{equation}
        \left. \frac{\partial \ell(\sigma^2, \eta)}{\partial \eta} \right|_{\hat{\sigma}^2(\eta)} = 
        -\frac{1}{2} \left( \trace(\tens{M}_{1,\eta}) - \hat{\sigma}^{-2}(\eta) \vect{z}^{\intercal} \tens{M}_{1,\eta}^2 \vect{z} \right). \label{eq:der-L-eta-alt}
    \end{equation}
    The above is also the total derivative in \eqref{eq:total-d-l}, and concludes \eqref{eq:der-ell-eta}.

    To find the second partial derivative with respect to \(\eta\), set \(k = 2\) in \eqref{eq:terms-der-k-L-eta} and apply into \eqref{eq:der-k-L} to obtain
    \begin{equation}
        \left. \frac{\partial^2 \ell(\sigma^2,\eta)}{\partial \eta^2} \right|_{\hat{\sigma}^2(\eta)} = \frac{1}{2} \left( \trace(\tens{M}_{1,\eta}^2) - 2\hat{\sigma}^{-2}(\eta) \vect{z}^{\intercal} \tens{M}_{1,\eta}^3 \vect{z} \right). 
        \label{eq:partial1-dd-l}
    \end{equation}
    The second partial derivative with respect to \(\sigma^2\) is directly obtained from \eqref{eq:der2-L-sigma2} and substituting \eqref{eq:sigma-hat} as
    \begin{equation}
        \left. \frac{\partial^2 \ell(\sigma^2,\eta)}{\partial (\sigma^2)^2} \right|_{\hat{\sigma}^2(\eta)} = -\frac{n-m}{2 (\hat{\sigma}^2(\eta))^2}.
            \label{eq:partial2-dd-l}
    \end{equation}
    The mixed second derivative is also obtained by taking the derivative of \eqref{eq:der-L-eta-alt} with respect to \(\hat{\sigma}^2\) as
    \begin{equation}
        \left. \frac{\partial^2 \ell(\sigma^2,\eta)}{\partial \sigma^2 \partial \eta} \right|_{\hat{\sigma}^2(\eta)} = -\frac{1}{2 (\hat{\sigma}^2(\eta))^2} \vect{z}^{\intercal} \tens{M}_{1,\eta}^2 \vect{z}.
        \label{eq:mixed-dd-l}
    \end{equation}
    Also, by \Cref{lem:der-M} for \(\theta = \eta\), we have \(\dot{\tens{M}}_{1,\eta} = -\tens{M}_{1,\eta}^2\). Thus, from \eqref{eq:sigma-hat}, we obtain,
    \begin{equation}
        \frac{\mathrm{d} \hat{\sigma}^2(\eta)}{\mathrm{d} \eta} = -\frac{1}{n-m} \vect{z}^{\intercal} \tens{M}_{1,\eta}^2 \vect{z}. \label{eq:sigma-hat-der}
    \end{equation}
    By substituting \eqref{eq:partial1-dd-l}, \eqref{eq:partial2-dd-l}, \eqref{eq:mixed-dd-l}, and \eqref{eq:sigma-hat-der} into \eqref{eq:total-dd-l} and upon rearrangement, the second total derivative in \eqref{eq:der2-ell-eta} is obtained.
\end{proof}

\begin{theorem}[Estimation of \(\eta\)] \label[theorem]{thm:eta}
    Suppose \(\vect{z} \notin \range(\tens{\X})\). Let \(\ell_{\hat{\sigma}^2(\eta)}(\eta)\) be defined as in \Cref{prop:der-eta}. If \(\hat{\eta}\) satisfies
    \begin{subequations}
        \begin{equation}
            \vect{z}^{\intercal} \tens{G}_{\hat{\eta}} \vect{z} = 0,
            \qquad \text{subject to}
            \qquad
            \vect{z}^{\intercal} \tens{H}_{\hat{\eta}} \vect{z} < 0,
            \label{eq:der-L-eta}
        \end{equation}
        where
        \begin{align}
            \tens{G}_{\eta} &\coloneqq \frac{\trace(\tens{M}_{1,\eta})}{n-m} \tens{M}_{1,\eta} - \tens{M}_{1,\eta}^2, \label{eq:G} \\
            \tens{H}_{\eta} &\coloneqq \left( \frac{\trace(\tens{M}_{1,\eta}^2)}{n-m} + \left( \frac{\trace(\tens{M}_{1,\eta})}{n-m} \right)^2 \right) \tens{M}_{1,\eta} - 2 \tens{M}_{1,\eta}^3, \label{eq:H}
        \end{align}
    \end{subequations}
    then, \(\hat{\eta}\) is a local maximum of \(\ell_{\hat{\sigma}^2(\eta)}(\eta)\).
\end{theorem}

\begin{proof}
    We look for a point \(\hat{\eta}\) at which the first derivative vanishes and the second derivative is negative. By multiplying the first total derivative in \Cref{prop:der-eta} by \(\hat{\sigma}^2(\eta)\) and substitute \(\hat{\sigma}^2(\eta)\) from \Cref{thm:sigma} we obtain 
    \begin{equation} \label{eq:der1-G}
        -2 \hat{\sigma}^2(\eta) \frac{\mathrm{d} \ell_{\hat{\sigma}^{2}(\eta)}(\eta)}{\mathrm{d} \eta} = \vect{z}^{\intercal} \tens{G}_{\eta} \vect{z}.
    \end{equation}
    But \(\hat{\sigma}^2(\eta) \neq 0\) since we assumed \(\vect{z} \notin \range(\tens{\X})\). Thus, the first derivative vanishes at some point \(\hat{\eta}\) whenever \(\vect{z}^{\intercal} \tens{G}_{\hat{\eta}} \vect{z}\) vanishes.

    We obtain the second condition as follows. At \(\hat{\eta}\), the zero first-derivative implies,
    \begin{equation}
        \vect{z}^{\intercal} \tens{M}_{1,\hat{\eta}}^2 \vect{z} = \frac{\trace(\tens{M}_{1,\hat{\eta}})}{n-m} \vect{z}^{\intercal} \tens{M}_{1,\hat{\eta}} \vect{z}. \label{eq:zM2z}
    \end{equation}
    We multiply the second total derivative in \Cref{prop:der-eta} by \(\hat{\sigma}^2(\hat{\eta})\), substitute \(\hat{\sigma}^2(\hat{\eta})\) from \Cref{thm:sigma}, also substitute \(\vect{z}^{\intercal} \tens{M}_{1,\hat{\eta}}^2 \vect{z}\) therein using \eqref{eq:zM2z}, to obtain,
    \begin{equation} \label{eq:der2-H} 
        2 \hat{\sigma}^2(\hat{\eta}) \left. \frac{\mathrm{d}^2 \ell_{\hat{\sigma}^2(\eta)}(\eta)}{\mathrm{d} \eta^2} \right|_{\hat{\eta}} = \vect{z}^{\intercal} \tens{H}_{\hat{\eta}} \vect{z}.
    \end{equation}
    Thus, the second derivative is negative whenever \(\vect{z}^{\intercal} \tens{H}_{\hat{\eta}} \vect{z} < 0\).
\end{proof}

A practical algorithm using \Cref{thm:sigma} and \Cref{thm:eta} is as follows. Let \(c \ll 1 \ll C\) be two reasonably chosen thresholds for \(\hat{\eta}\). Once an estimate of \(\hat{\eta}\) is found by \Cref{thm:eta}, three cases can hold. Firstly, if \(\hat{\eta} < c\), the noise variance is considered negligible compared to the error variance; set \(\sigma_0^2 = 0\) and calculate \(\sigma^2\) from \eqref{eq:variance-trivial}. Secondly, if \(\hat{\eta} > C\), the error variance is considered negligible compared to the noise variance; set \(\sigma^2 = 0\) and calculate \(\sigma_0^2\) from \eqref{eq:nugget-trivial}. Finally, if \(c \leq \hat{\eta} \leq C\), compute \(\tens{M}_{1,\hat{\eta}}\) and use \Cref{thm:sigma} to find \(\hat{\sigma}^2\); the noise variance is then found by \(\hat{\sigma}_0^2 = \hat{\eta} \hat{\sigma}^2\).

The computational advantage of finding hyperparameters using the above method is significant compared to maximizing \(\ell(\sigma^2,\sigma_0^2)\) directly in the space of two hyperparameters \((\sigma^2,\sigma_0^2)\). This is because both relations in \eqref{eq:der-L-eta} are independent of \(\hat{\sigma}^2\), thus, \(\hat{\eta}\) can be exclusively found. In other words, the dimension of the space of hyperparameter search is reduced.


\section{Bounds and Asymptotic Properties} \label{sec:properties}

Solving \eqref{eq:der-L-eta} generally only guarantees that \(\hat{\eta}\) is a local maximum of the likelihood function. We of course seek the global maximum. In \Cref{prop:sign-indefinite} below we show that \(\tens{G}_{\eta}\) and \(\tens{H}_{\eta}\) are generally sign-indefinite, which implies there could exist an arbitrary number of local maxima depending on the data set. However, we show in this section that deriving bounds for \(\ell\) and its derivatives, and the asymptotic behavior of its derivative, can be useful for locating the global maximum \(\hat{\eta}\).


\begin{proposition}[Sign-Indefiniteness of the Derivatives of \(\ell\)] \label[proposition]{prop:sign-indefinite}
    Matrices \(\tens{G}_{\eta}\) and \(\tens{H}_{\eta}\) are either unconditionally sign-indefinite, or identically zero.
\end{proposition}

\begin{proof}
    We will show the eigenvalues of \(\tens{G}_{\eta}\) and \(\tens{H}_{\eta}\) are either always sign-indefinite for any \(\eta\), or all zero.

    \begin{enumerate}[leftmargin=*,align=left,label*=\emph{Step (\roman*).},ref=step (\roman*),wide]
        \item\label{item:gh1} Let \(\eigM_i\) denote the eigenvalues of \(\tens{M}_{1,\eta}\) arranged in ascending order. Recall that \(\tens{M}_{1,\eta} = \tens{K}_{\eta}^{-1} \tens{P}_{\eta}\), where \(\tens{K}_{\eta}^{-1}\) is full rank and positive-definite. Also, recall that \(\tens{P}_{\eta}\) is a projection matrix with the kernel space \(\range(\tens{\X})\) with the dimension \(m\). Hence, \(m\) eigenvalues of \(\tens{P}_{\eta}\), and accordingly, the first \(m\) eigenvalues of \(\tens{M}_{1,\eta}\), are zero, \ie \(\eigM_1 = \dots = \eigM_m = 0\), whilst the rest of the eigenvalues of \(\tens{M}_{1,\eta}\) are positive.
        \item Define,
    \begin{subequations}
    \begin{align}
        \overbar{\eigM} &\coloneqq \frac{\trace(\tens{M}_{1,\eta})}{n-m} = \frac{1}{n-m} \sum_{k=m+1}^n \eigM_k, \label{eq:mu-bar-1} \\
        \overbar{\eigM^2} &\coloneqq \frac{\trace(\tens{M}_{1,\eta}^2)}{n-m} = \frac{1}{n-m}\sum_{k=m+1}^n \eigM_k^2, \label{eq:mu-bar-2}
    \end{align}
    \end{subequations}
    which they represent the mean and square mean of the non-zero eigenvalues of \(\tens{M}_{1,\eta}\). It holds,
    \begin{equation}
        \eigM_{m+1} \leq \overbar{\eigM} \leq \eigM_n,
        \qquad \text{and} \qquad
        \eigM_{m+1}^2 \leq \overbar{\eigM^2} \leq \eigM_n^2.
        \label{eq:mu-bar-ineq}
    \end{equation}
    The equalities in the above hold only if \(\eigM_{m+1} = \dots = \eigM_n\).

    \item Let \(\gamma_i\) and \(\vartheta_i\) denote the eigenvalues of \(\tens{G}_{\eta}\) and \(\tens{H}_{\eta}\), respectively. From the spectral decomposition of symmetric matrices \(\tens{M}_{1,\eta}\), \(\tens{G}_{\eta}\), and \(\tens{H}_{\eta}\) in the relations \eqref{eq:G} and \eqref{eq:H}, we have,
    \begin{subequations}
    \begin{align}
        \gamma_i &= \eigM_i \left(\overbar{\eigM} - \eigM_i \right), \\
        \vartheta_i &= \eigM_i \left( \overbar{\eigM^2} + \overbar{\eigM}^2 - 2\eigM_i^2 \right).
    \end{align}
    \end{subequations}
    Similar to \(\eigM_i\), we have \(\gamma_1 = \dots = \gamma_m = 0\) and \(\vartheta_1 = \dots = \vartheta_m = 0\). When \(\eigM_{m+1} = \dots = \eigM_n\), all \(\gamma_i\) and \(\vartheta_i\), and hence \(\tens{G}_{\eta}\) and \(\tens{H}_{\eta}\), become trivially zero. In the non-trivial case, because of the inequalities in \eqref{eq:mu-bar-ineq}, both \(\gamma_i\) and \(\vartheta_i\) contain positive and negative values at \(i > m\), and concludes the sign-indefiniteness of matrices.
    \end{enumerate}
\end{proof}

We exclude the trivial case when \(\tens{G}_{\eta}\) and \(\tens{H}_{\eta}\) are identically zero as \(\ell\) becomes constant. By \Cref{prop:sign-indefinite}, the quadratic forms in \eqref{eq:der-L-eta} always remain sign-indefinite, and any prior knowledge on the local maxima of \(\ell\) cannot be inferred without knowledge of the data \(\vect{z}\). Therefore, it is not known a priori  if there exists one or multiple local minima, or what is a reasonable range of \(\eta\) to locate the root(s) of the derivative of \(\ell\), if any exists. Next we derive formulas for the bounds of \(\ell\), its derivatives, and the asymptotic behavior of its derivative, as these relations will be useful for determining an initial guess the extremum of \(\ell\).


\subsection{Bounds of \texorpdfstring{\(\ell\)}{L} and its Derivatives} \label{sec:bounds}

\begin{proposition}[Bounds on the Derivatives of \(\ell\)] \label[proposition]{prop:der-bound}
    Suppose \(\vect{z} \notin \range(\tens{\X})\). Let \(\lambda_1\) and \(\lambda_n\) respectively denote the smallest and the largest eigenvalues of the correlation matrix \(\tens{K}\). Then,
    \begin{subequations} \label{eq:der-bounds}
    \begin{equation}
        \left| \frac{\mathrm{d} \ell_{\hat{\sigma}^2(\eta)}(\eta)}{\mathrm{d} \eta} \right| \leq \frac{n-m}{2} \left( \frac{1}{\lambda_1 + \eta} - \frac{1}{\lambda_n + \eta} \right), \label{eq:ldot-bound}
    \end{equation}
    and
    \begin{equation}
        \left| \frac{\mathrm{d}^2 \ell_{\hat{\sigma}^2(\eta)}(\eta)}{\mathrm{d} \eta^2} \right| \leq (n-m) \left( \frac{1}{(\lambda_1 + \eta)^2} - \frac{1}{(\lambda_n + \eta)^2} \right). \label{eq:lddot-bound}
    \end{equation}
\end{subequations}
\end{proposition}

\begin{proof}
    See \S \ref{prf:der-bound}.
\end{proof}

The bounds on the derivatives in \Cref{prop:der-bound} are not sharp, and may not be directly employed for practical use. However, they hint to the fast decay of the variations of \(\ell\) at large values of \(\eta\). Also, they indicate the role of the largest and smallest eigenvalues of \(\tens{K}\) based on the bounds in relations \eqref{eq:der-bounds}. Namely, at \(\eta \gg \mathcal{O}(\lambda_n)\), the variation of \(\ell\) is relatively insignificant. On the other end, the bound on derivatives at \(\eta \ll \lambda_1\) are almost constant, although the derivatives of \(\ell\) may not follow their bound closely. Nonetheless, in practice, we may roughly expect the derivative of \(\ell\) varies less than an order of magnitude at \(\eta \ll \lambda_1\). This coarse analysis suggests an eigenvalue-dependent range, such as \(\eta \in [\mathcal{O}(\lambda_1),\mathcal{O}(\lambda_n)]\), in which we can search for the global maxima of \(\ell\).

Another implication of \Cref{prop:der-bound} can be deduced as follows.

\begin{corollary}[Bounds on \(\ell\)] \label[corollary]{cor:l-bound}
    Suppose \(\vect{z} \notin \range(\tens{\X})\). Let \(\lambda_1\) and \(\lambda_n\) denote the smallest and the largest eigenvalues of the correlation matrix \(\tens{K}\). Then,
    \begin{equation}
        \left| \ell_{\hat{\sigma}^2(\eta)}(\eta) - \ell_{\hat{\sigma}^2(\eta')}(\eta') \right| \leq \frac{n-m}{2} \log \left( \frac{\lambda_1+\eta}{\lambda_n + \eta} \; \frac{\lambda_n + \eta'}{\lambda_1 + \eta'} \right). \label{eq:l-bound}
    \end{equation}
\end{corollary}

\begin{proof}
    Taking the definite integral of \eqref{eq:ldot-bound} concludes \eqref{eq:l-bound}.
\end{proof}

\begin{remark}
    While the bound of \(\ell\) can be obtained by taking the integral of the bound on its derivative (as done in \Cref{cor:l-bound}), the opposite cannot be inferred. That is, taking the derivative of the bound of \(\ell\) in \eqref{eq:l-bound} does not imply the bound of its first derivative in \eqref{eq:ldot-bound}.
\end{remark}

The bound of \(\ell\) in \Cref{cor:l-bound} may be employed to narrow the range of \(\eta\) in searching the global maxima. For example, if \(\ell\) at \(\eta' = 0\) is known, by another evaluation of \(\ell\) at a second location \(\eta''\), we can find \(0 \leq c \leq \eta''\) from the monotonically increasing bound in \eqref{eq:l-bound} so that \(\ell_{\hat{\sigma}^2(\eta)}(\eta) < \ell_{\hat{\sigma}^2(\eta'')}(\eta'') \) for all \(\eta < c\), and hence, narrow the search to \(\eta > c\) thereafter.


\subsection{Asymptote of the First Derivative of \texorpdfstring{\(\ell\)}{L}} \label{sec:asymptote}

As argued above, it is reasonable to search for the maxima of \(\ell\) in \(\eta \in [\mathcal{O}(\lambda_1),\mathcal{O}(\lambda_n)]\). However, outside this range will be left unexplored. Fortunately, we can check in advance the existence of local maxima in \(\eta \gg \lambda_n\), and roughly at \(\eta \in [\mathcal{O}(\lambda_n),\infty)\), using an asymptotic approximation of \(\ell\) at large values of \(\eta\). The numerical evaluation of the approximated formulation is inexpensive and can be employed readily before the main numerical optimization. Namely, we will utilize an asymptotic relation for the first total derivative of \(\ell\), which is derived in \Cref{prop:asymptote}. 

\begin{lemma}[Asymptote of \(\tens{M}_{1,\eta}\)] \label[lemma]{lem:M-asym}
    Let \(\lambda_n\) denote the largest eigenvalue of \(\tens{K}\). Then, at \(\eta \gg \lambda_n\),
    \begin{equation}
        \tens{M}_{1,\eta} = \frac{1}{\eta} \tens{Q} \left( \tens{I} - \frac{1}{\eta} \tens{N} + \frac{1}{\eta^2} \tens{N}^2 \right) + \mathcal{O}\left(\eta^{-4} \lambda_n^3 \right), \label{eq:M-asymptote}
    \end{equation}
    where \(\tens{I}\) is the \(n \times n\) identity matrix, \(\tens{N} \coloneqq \tens{K} \tens{Q}\), and,
    \begin{equation}
        \tens{Q} \coloneqq \tens{I} - \tens{\X} \left( \tens{\X}^{\intercal} \tens{\X} \right)^{-1} \tens{\X}^{\intercal}. \label{eq:Q}
    \end{equation}
\end{lemma}

\begin{proof}
    See \S \ref{prf:M-asym}.
\end{proof}

\begin{proposition}[Asymptote of the Derivative of \(\ell\)] \label[proposition]{prop:asymptote}
    Suppose \(\vect{z} \notin \range(\tens{\X})\). Let \(\lambda_n\) denote the largest eigenvalue of \(\tens{K}\). Then, at \(\eta \gg \lambda_n\),
    \begin{equation}
        \frac{\mathrm{d} \ell_{\hat{\sigma}^2(\eta)}(\eta)}{\mathrm{d} \eta} = -\frac{n-m}{2} \frac{1}{\eta^2} \left( a_0 + \frac{a_1}{\eta} + \frac{a_2}{\eta^2} + \frac{a_3}{\eta^3} \right) + \mathcal{O}\left(\eta^{-6} \lambda_n^5\right), \label{eq:asymptote}
    \end{equation}
    with the coefficients \(a_i \coloneqq \check{\vect{z}}^{\intercal} \tens{A}_i \check{\vect{z}}\), \(i = 0,\dots,3\), where \(\check{\vect{z}} \coloneqq \vect{z} / \| \vect{z} \|_{\tens{Q}}\), and,
    \begin{subequations} \label{eq:Ai}
    \begin{align}
        \tens{A}_0 &\coloneqq -\tens{Q} \left( \frac{\trace(\tens{N})}{n-m} \tens{I} - \tens{N} \right), \\
        \tens{A}_1 &\coloneqq +\tens{Q} \left( \frac{\trace(\tens{N}^2)}{n-m} \tens{I} + \frac{\trace(\tens{N})}{n-m} \tens{N} - 2 \tens{N}^2 \right), \\
        \tens{A}_2 &\coloneqq -\tens{Q}\left( \frac{\trace(\tens{N}^2)}{n-m} \tens{N} + \frac{\trace(\tens{N})}{n-m} \tens{N}^2 - 2\tens{N}^3 \right), \\
        \tens{A}_3 &\coloneqq +\tens{Q}\left( \frac{\trace(\tens{N}^2)}{n-m} \tens{N}^2 -\tens{N}^4 \right).
    \end{align}
    \end{subequations}
\end{proposition}

\begin{proof}
    The first total derivative of \(\ell_{\hat{\sigma}^2(\eta)}(\eta)\) from \eqref{eq:der1-G} is
    \begin{equation}
        \frac{\mathrm{d} \ell_{\hat{\sigma}^2(\eta)}(\eta)}{\mathrm{d} \eta} = -\frac{n-m}{2} \frac{\vect{z}^{\intercal} \tens{G} \vect{z}}{\vect{z}^{\intercal} \tens{M}_{1,\eta} \vect{z}},
        \label{eq:d-ell-den-nom}
    \end{equation}
    where \(\tens{G}\) is defined in \Cref{thm:eta}.  We approximate the denominator of \eqref{eq:d-ell-den-nom} by only the first term in the asymptotic relation of \(\tens{M}_{1,\eta}\) in \Cref{lem:M-asym}, \ie
    \begin{equation*}
        \vect{z}^{\intercal} \tens{M}_{1,\eta} \vect{z} \approx \frac{1}{\eta} \vect{z}^{\intercal} \tens{Q} \vect{z}.
    \end{equation*}
    Recall that the quadratic term in the above can also be written as the norm \(\| \vect{z} \|_{\tens{Q}}^2\). By defining \(\check{\vect{z}} = \vect{z} / \| \vect{z} \|_{\tens{Q}}\), we absorb the denominator of \eqref{eq:d-ell-den-nom} into its numerator by
    \begin{equation}
        \frac{\mathrm{d} \ell_{\hat{\sigma}^2(\eta)}(\eta)}{\mathrm{d} \eta} = -\frac{n-m}{2} \eta \check{\vect{z}}^{\intercal} \tens{G} \check{\vect{z}}.
        \label{eq:d-ell-den-nom-2}
    \end{equation} 
    In the following steps, we derive an asymptote for each of the terms involving in \(\tens{G}\) as \(\eta^{-1}\) shrinks.

    \begin{enumerate}[leftmargin=*,align=left,label*=\emph{Step (\roman*).},ref=step (\roman*),wide]
        \item\label{item:dl-asym1} Recall from \eqref{eq:Q} that \(\tens{Q} = \tens{I} - \tens{\X} \left( \tens{\X}^{\intercal} \tens{\X} \right)^{-1} \tens{\X}^{\intercal}\). By the cyclic property of trace operation, we have,
            \begin{align*}
                \trace(\tens{Q}) &= \trace(\tens{I}_{n \times n}) - \trace\left( \left( \tens{\X}^{\intercal} \tens{\X} \right)^{-1} \tens{\X}^{\intercal} \tens{\X} \right) \\
                &= \trace(\tens{I}_{n \times n}) - \trace(\tens{I}_{m \times m}) \\
                &= n-m.
            \end{align*}
            Also, \(\tens{Q}^2 = \tens{Q}\), because the projection matrix \(\tens{Q}\) is idempotent. Hence,
            \begin{equation*}
                \trace(\tens{Q} \tens{N}) = \trace(\tens{Q} \tens{K} \tens{Q}) = \trace(\tens{K} \tens{Q}^2) = \trace(\tens{K} \tens{Q}) = \trace(\tens{N}).
            \end{equation*}
            Similarly, \(\trace(\tens{Q} \tens{N}^2) = \trace(\tens{N}^2)\). Overall, from \Cref{lem:M-asym} we have,
            \begin{equation*}
                \frac{\trace(\tens{M}_{1,\eta})}{n-m} = \frac{1}{\eta} \left( 1 - \frac{1}{\eta} \frac{\trace(\tens{N})}{n-m} + \frac{1}{\eta^2} \frac{\trace(\tens{N}^2)}{n-m} \right) + \mathcal{O}\left( \eta^{-4} \lambda_n^3 \right).
            \end{equation*}
            From \Cref{lem:M-asym} and the above relation, we have,
            \begin{equation*}
                \begin{split}
                    \frac{\trace(\tens{M}_{1,\eta})}{n-m} \tens{M_{1,\eta}} = \frac{1}{\eta^2} \tens{Q} \bigg[ \tens{I} &- \frac{1}{\eta} \left( \frac{\trace(\tens{N})}{n-m} \tens{I} + \tens{N} \right) \\
                    &+ \frac{1}{\eta^2} \left( \frac{\trace(\tens{N}^2)}{n-m} \tens{I} + \frac{\trace(\tens{N})}{n-m} \tens{N} + \tens{N}^2 \right) \\
                    &-\frac{1}{\eta^3} \left( \frac{\trace(\tens{N}^2)}{n-m} \tens{N} + \frac{\trace(\tens{N})}{n-m} \tens{N}^2 \right) \\
                    &+ \frac{1}{\eta^4} \left( \frac{\trace(\tens{N}^2)}{n-m} \tens{N}^2 \right) \bigg] + \mathcal{O}\left( \eta^{-7} \lambda_n^5 \right).
                \end{split}
            \end{equation*}
            Note that, in the above, we have not omitted higher-order terms.
           
        \item\label{item:dl-asym2} Also, \(\tens{M}_{1,\eta}^2\) is calculated from \Cref{lem:M-asym} by,
            \begin{equation*}
                \begin{split}
                    \tens{M}_{1,\eta}^2 = \tens{M}_{1,\eta} \tens{M}_{1,\eta}^{\intercal} &= \frac{1}{\eta^2} \tens{Q} \bigg[ \tens{I} - \frac{1}{\eta} \left( \tens{N} + \tens{N}^{\intercal} \right) + \frac{1}{\eta} \left( \tens{N}^2 + \tens{N} \tens{N}^{\intercal} + (\tens{N}^{\intercal})^2 \right) \\
                &-\frac{1}{\eta^3} \tens{N} \left( \tens{N} + \tens{N}^{\intercal} \right) \tens{N}^{\intercal} + \frac{1}{\eta^4} \tens{N}^2 (\tens{N}^{\intercal})^2 \bigg] \tens{Q} + \mathcal{O}\left( \eta^{-7} \lambda^5 \right).
            \end{split}
            \end{equation*}
            We can simplify the above by writing \(\tens{N} = \tens{K} \tens{Q}\) and applying \(\tens{Q}^2 = \tens{Q}\), which leads to,
            \begin{equation*}
                \tens{M}_{1,\eta}^2 = \frac{1}{\eta^2} \tens{Q} \bigg[
                \tens{I} - \frac{2}{\eta} \tens{N} + \frac{3}{\eta^2} \tens{N}^2 - \frac{2}{\eta^3} \tens{N}^3 + \frac{1}{\eta^4} \tens{N}^4 \bigg] + \mathcal{O}\left( \eta^{-7} \lambda_n^5 \right).
            \end{equation*}

        \item\label{item:dl-asym4} By subtracting the results of \ref{item:dl-asym1} from \ref{item:dl-asym2}, the matrix \(\tens{G}\) in \eqref{eq:d-ell-den-nom-2} becomes
            \begin{equation*}
                \tens{G} = \frac{1}{\eta^3} \tens{Q} \left( \tens{A}_0 + \frac{1}{\eta} \tens{A}_1 + \frac{1}{\eta^2} \tens{A}_2 + \frac{1}{\eta^3} \tens{A}_3 \right) + \mathcal{O}\left( \eta^{-7} \lambda_n^5 \right),
            \end{equation*}
            with \(\tens{A}_i\), \(i = 0,\dots,3\), defined in \eqref{eq:Ai}. Combining the above with \eqref{eq:d-ell-den-nom-2} concludes \eqref{eq:asymptote}.
    \end{enumerate}
\end{proof}

A few remarks on \Cref{prop:asymptote} are as follows.

\begin{remark}[Alternative Representation]
    The matrices \(\tens{A}_2\) and \(\tens{A}_3\) in \Cref{prop:asymptote} can also be computed from \(\tens{A}_0\) and \(\tens{A}_1\) by \(\tens{A}_2 = -\tens{A}_1 \tens{N}\) and \(\tens{A}_3 = (\tens{A}_1 + \tens{A}_0 \tens{N}) \tens{N}^2 \).
\end{remark}

\begin{remark}[Accuracy of Asymptote]
    In the proof of \Cref{prop:asymptote}, we kept all higher-order terms in the approximation of the numerator of \eqref{eq:d-ell-den-nom}, whereas, we approximated its denominator only by the first dominant term. Our motivation for such is that, in practice, we are only interested in the change of the sign and the roots of the derivative of \(\ell_{\hat{\sigma}^2(\eta)}(\eta)\), which are determined by its numerator. Also, by a slight miss-representation, the accuracy of the denominator, \ie \(\mathcal{O}\left( \eta^{-1} \lambda_n \right)\), is not incorporated in the overall accuracy of \eqref{eq:asymptote}. That is, the \(\mathcal{O}\left( \eta^{-6} \lambda_n^5 \right)\) in \eqref{eq:asymptote} only represents the accuracy of the numerator of the derivative.
\end{remark}

\begin{remark}[Validity of Asymptote] \label[remark]{rem:valid-asym}
    In practice, the condition \(\eta \gg \lambda_n\) may be loosened to \(\eta > \mathcal{O}(\lambda_n)\), as in the latter range, we can still locate a rough estimate on the roots of the derivative of \(\ell\) (see example in \S \ref{sec:max-LogL}, in particular, \Cref{fig:LogLikelihoodDerivative}).
\end{remark}

\begin{remark}[Case of Large Data]
    In many applications, such as the example we will provide in \S \ref{sec:example}, the size of the data may be significantly larger than the number of basis functions. We can leverage the assumption \(n \gg m\) to simply the asymptotic relation in \Cref{prop:asymptote} as follows. Recall from \S \ref{prf:M-asym} that  \(\tens{Q} = \tens{I} - \tens{Q}_{\perp}\) where \(\tens{Q}_{\perp} = \tens{\X} (\tens{\X}^{\intercal} \tens{\X})^{-1} \tens{\X}^{\intercal} \). Since \(\tens{Q}_{\perp}\) is a symmetric projection matrix of the rank \(m\), it has exactly \(m\) nonzero eigenvalues and all are equal to \(1\). Hence, its Frobenius squared norm is \(\| \tens{Q}_{\perp} \|_F^2 = m\), which is the square of the sum of the \(n^2\) elements of \(\tens{Q}_{\perp}\). When \(n \gg m\), the entries of the matrix \(\tens{Q}_{\perp}\) become significantly small, and we can approximate,
    \begin{equation*}
        \tens{Q} \approx \tens{I}_{n \times n}, \qquad \text{hence,} \qquad \tens{N} \approx \tens{K}.
    \end{equation*}
    Also,
    \begin{align*}
        \trace(\tens{N}) &\approx \trace(\tens{K}) = n,\\
        \trace(\tens{N}^2) &\approx \trace(\tens{K}^2) = \| \tens{K} \|_{F}^2,
    \end{align*}
    since the diagonals of the correlation matrix \(\tens{K}\) are all \(1\). Also, \(\| \tens{K} \|_F\) denotes the Frobenius norm of \(\tens{K}\), which can be computed inexpensively.
\end{remark}

\begin{remark}[Lower Order Asymptote] \label[remark]{rem:lower-asym}
    The asymptotic relation of \Cref{prop:asymptote} is derived based on the second-order approximation of \(\tens{M}_{1,\eta}\) in \Cref{lem:M-asym}. A first-order asymptote of \(\tens{M}_{1,\eta}\), however, omits both \(a_2 \eta^{-2}\) and \(a_3 \eta^{-3}\) terms in \eqref{eq:asymptote}. That is, if one wishes to employ a lower order approximation, both of \(a_2 \eta^{-2}\) and \(a_3 \eta^{-3}\) terms should be dropped, as omitting only the latter term leads to an incorrect simplification.
\end{remark}

The asymptotic relation of \Cref{prop:asymptote} is quite useful in practice. The real roots of the polynomial 
\begin{subequations}
\begin{equation}
    a_0 \eta^3 + a_1 \eta^2 + a_2 \eta + a_3 = 0, \label{eq:root_2}
\end{equation}
(if the second-order approximation is used), or
\begin{equation}
    a_0 \eta + a_1 = 0, \label{eq:root_1}
\end{equation}
\end{subequations}
(if the first-order approximation is used), can readily provide an approximation for the local extrema of \(\ell\) at large values of \(\eta\). Also, we note that the evaluation of \(a_i\) in \Cref{prop:asymptote} is fairly inexpensive. We present the applicability of the asymptotic approximation with our numerical example in \S \ref{sec:max-LogL}. 


\section{Implementation Considerations} \label{sec:efficient}

Evaluations of \(\ell\) and its derivative(s) in \Cref{thm:eta} can be expensive, particularly for large data. Below, we overview a few implementation considerations to compute the first quadratic form in \eqref{eq:der-L-eta} (which effectively computes the first derivative of \(\ell\) via \eqref{eq:der1-G}). Numerical implementation of the second quadratic form in \eqref{eq:der-L-eta} (which effectively computes the second derivative of \(\ell\) via \eqref{eq:der2-H}) follows a similar approach, but we omit this discussion for brevity.


\subsection{Preliminary Numerical Improvements} \label{sec:prelim-numerical}

The estimation of \(\hat{\eta}\) by \Cref{thm:eta} requires several numerical evaluations of the quadratic form
\begin{equation*}
    q(\eta) \coloneqq \vect{z}^{\intercal} \left( \frac{\trace(\tens{M}_{1,\eta})}{n-m} \tens{M}_{1,\eta} - \tens{M}^2_{1,\eta} \right) \vect{z},
\end{equation*}
from \eqref{eq:der-L-eta}. Define \(\vect{w} \coloneqq \tens{M}_{1,\eta} \vect{z}\). If \(\vect{w}\) and the trace of \(\tens{M}_{1,\eta}\) are known, evaluating \(q\) is a simple algebraic operation by
\begin{equation*}
    q(\eta) = \frac{\trace(\tens{M}_{1,\eta})}{n-m} \vect{z}^{\intercal} \vect{w} - \| \vect{w} \|^2.
\end{equation*}
The vector \(\vect{w}\) can be obtained as follows. Recall that \(\tens{M}_{1,\eta}\) is defined by \eqref{eq:M1}. Define \(\vect{u} \coloneqq \tens{K}_{\eta}^{-1} \vect{z}\) and \(\tens{Y} \coloneqq \tens{K}_{\eta}^{-1} \tens{\X}\). Thus,
\begin{equation}
    \vect{w} = \vect{u} - \tens{Y} (\tens{\X}^{\intercal}\tens{Y})^{-1} \tens{Y}^{\intercal} \vect{z}. \label{eq:w}
\end{equation}
To circumvent the expensive (\ie \(\mathcal{O}(n^3)\)) numerical evaluation of \(\tens{K}_{\eta}^{-1}\) directly, we solve \(\vect{u}\) and \(\tens{Y}\) from the linear systems \(\tens{K}_{\eta} \vect{u} = \vect{z}\) and \(\tens{K}_{\eta} \tens{Y} = \tens{\X}\), since, linear systems of symmetric and positive-definite matrix \(\tens{K}_{\eta}\) can be solved very efficiently, such as by the conjugate gradient (CG) method with incomplete Cholesky factorization preconditioning (see \citet[\S 6.7 and \S 10.8.2]{SAAD-2003} and \citet[\S 10.2]{GOLUB-1996}. The computational complexity of the inexact CG method depends on matrix-vector multiplication, which requires \(\mathcal{O}(\operatorname{nnz}(\tens{K}))\) operations, where  \(\operatorname{nnz}(\tens{K})\) denotes the number of non-zero elements of \(\tens{K}\). Note that \((\tens{\X}^{\intercal} \tens{Y})\) is an \(m \times m\) matrix where \(m\) is often relatively small. Thus, inverting \(\tens{\X}^{\intercal} \tens{Y}\) directly is inexpensive.



Additionally, the trace of \(\tens{M}_{1,\eta}\) can be obtained by
\begin{equation*}
    \trace(\tens{M}_{1,\eta}) = \trace(\tens{K}_{\eta}^{-1}) - \trace((\tens{\X}^{\intercal} \tens{Y})^{-1} (\tens{Y}^{\intercal} \tens{Y})).
\end{equation*}
For the second term on the right, we have used the cyclic property of the trace operator. Both \((\tens{\X}^{\intercal} \tens{Y})^{-1}\) and \(\tens{Y}^{\intercal} \tens{Y}\) are \(m \times m\) matrices, and calculating the trace of their product is inexpensive. The main challenge is to evaluate \(\trace(\tens{K}_{\eta}^{-1})\), which we discuss in the next section.


\subsection{Computing the Trace of Matrix Inverse} \label{sec:pre-computation}

If \(\tens{K}\) is small enough to obtain all its eigenvalues, \(\lambda_i\), computing \(\trace(\tens{K}_{\eta}^{-1})\) for any \(\eta\) is immediate, by
\begin{equation}
    \trace(\tens{K}_{\eta}^{-1}) = \sum_{i = 1}^n \frac{1}{\lambda_i + \eta}. \label{eq:trace-lambda}
\end{equation}
In practice, this method is inefficient for large matrices as obtaining all eigenvalues of a matrix is \(\mathcal{O}(n^3)\) expensive. Another approach utilizes the Cholesky factorization of the symmetric and positive-definite matrix \(\tens{K}_{\eta} = \tens{L}_{\eta} \tens{L}_{\eta}^{\intercal}\), where \(\tens{L}_{\eta}\) is lower triangular. Then,
\begin{equation}
    \trace(\tens{K}_{\eta}^{-1}) = \trace(\tens{L}_{\eta}^{-\intercal} \tens{L}_{\eta}^{-1}) = \trace(\tens{L}_{\eta}^{-1} \tens{L}_{\eta}^{-\intercal}) = \| \tens{L}_{\eta}^{-1} \|^2_{F}, \label{eq:trace-chol}
\end{equation}
where \(\| \cdot \|_F\) is the Frobenius norm. In the second equality in the above, we used the cyclic property of the trace operator. For sparse matrices, there exist efficient methods to compute their Cholesky factorization (see \eg \citet[Ch. 4]{DAVIS-2006}). Also, the inverse of the sparse triangular matrix \(\tens{L}_{\eta}\) can be computed by \(\mathcal{O}(n^2)\) operations \cite[pp. 93-95]{STEWART-1998}.

The trace of the inverse of large and sparse matrices can also be computed by randomized estimators. The simplest of these methods is the stochastic trace estimator of \citet{HUTCHINSON-1990}. Another efficient method is the stochastic Lanczos quadrature algorithm by \citet{BAI-1996,BAI-1997,BAI-1999} and \citet{GOLUB-2009} that incorporates Golub-Kahn-Lanczos bi-diagonalization and Gauss quadrature.

In our application, the trace of \(\tens{K}_{\eta}^{-1}\) should be evaluated repeatedly during the numerical optimization procedure, where, only the hyperparameter \(\eta\) varies in the above-mentioned matrix. For this particular purpose, \cite{AMELI-2020} developed an efficient numerical technique by interpolating \(\eta \mapsto \trace(\tens{K}_{\eta}^{-1})\), and described in brief as follows.

Let \(\tau(\eta) \coloneqq \frac{1}{n} \trace(\tens{K}_{\eta}^{-1})\) and \(\tau_0 \coloneqq \tau(0)\). We pre-compute \(\tau_i \coloneqq \tau(\eta_i)\) for \(p\) interpolant points \(\eta_i\), \(i = 1,\dots,p\), using any of the methods mentioned before. Then, \(\tau(\eta)\) is interpolated by
\begin{equation}
    \frac{1}{\tau(\eta)} \approx \frac{1}{\tau_0} + \sum_{i = 0}^{p} w_i \varphi_i(\eta), \label{eq:approx}
\end{equation}
where \(\varphi_i\) are basis functions defined by
\begin{equation}
    \varphi_i(\eta) = \eta^{\frac{1}{i+1}}, \qquad i = 0,\dots,p, \label{eq:basis}
\end{equation}
and \(w_0 = 1\). The rest of the coefficients \(w_i\), \(i = 1,\dots,p\) are found by solving a linear system of \(p\) equations using a priori known values \(\tau_i\). The interpolation function \eqref{eq:approx} acknowledges the exact value of \(\tau(\eta)\) at \(\eta = 0\) and \(\eta = \eta_i\), and asymptotic to the true value of the function \(\tau(\eta)\) at \(\eta \to \infty\). When \(p = 0\), no interpolation point is introduced and the interpolation function \eqref{eq:approx} is proven to become a sharp upper bound for \(\tau(\eta)\). In practice, only a few interpolant points, \eg \(p < 5\), are sufficient to estimate \(\tau(\eta)\) in a large range of \(\eta\) with remarkable accuracy. However, if one wishes to employ many interpolant points, we suggest using an orthonormalized set of basis functions based on \eqref{eq:basis} to avoid an ill-conditioned system of equations in solving the weights \(w_i\) \citep{AMELI-2020}. 


\section{Numerical Example} \label{sec:example}

We examine our method through some examples. The data and the linear model are given in \S \ref{sec:data}. In \S \ref{sec:max-LogL}, we provide the details of the root-finding procedure for maximizing the marginal likelihood. In \S \ref{sec:various-noise}, we test our method for various noise levels and orders of the basis functions of the model. The performance and scalability of our method are examined in \S \ref{sec:scalability} for dense (\S \ref{sec:dense}) and sparse (\S \ref{sec:sparse}) correlation matrices. Lastly, in \S \ref{sec:optimal-parameters}, we seek the estimation of a broader set of covariance hyperparameters using the Mat\'{e}rn correlation function.


\subsection{Data and Model} \label{sec:data}

We generate sample data and a linear model as follows. Consider the mean function
\begin{equation}
    \mu(\vect{x}) = \sin(\pi x_1) + \sin(\pi x_2), \label{eq:mu-data}
\end{equation}
where \(\vect{x} \coloneqq (x_1,x_2)\) is in the domain \(\mathcal{D} = [0,1]^2\). Noise is introduced to the data by the spatially decorrelated Gaussian process \(\epsilon(\vect{x}) \sim \mathcal{N}(0,\sigma_0^2 )\). We set the noise level to \(\sigma_0 = 0.2\), thus, the noise-to-signal ratio is about \(0.3\). Later, in \S \ref{sec:various-noise}, we will examine various noise levels.

To represent the mean function by a linear model, we use the monomial basis functions of up to the \(q\)\textsuperscript{th} order as,
\begin{equation}
    \vect{\basis}(\vect{x}) = (1,x_1,x_2,x_1^2,x_1 x_2,x_2^2,\dots,x_1^q,x_1^{q-1}x_2, \dots, x_1 x_2^{q-1}, x_2^{q}). \label{eq:poly-basis}
\end{equation}
So, the number of basis functions is \(m = (q+1)(q+2)/2\). We mainly employ second-order basis functions, \ie \(q = 2\), which are suitable to model the mean function \(\mu\) compared to other polynomial orders. But, in \S \ref{sec:various-noise}, we will compare models with different polynomial orders.

We choose the isotropic exponential decay kernel for the spatial correlation function by
\begin{equation}
    K(\vect{x},\vect{x}'|\alpha) = \exp \left(-\frac{\| \vect{x} - \vect{x}' \|_2}{\alpha} \right), \label{eq:exp-decay}
\end{equation}
where \(\alpha\) is the decorrelation scale of the kernel. The above exponential decay kernel represents an Ornstein-Uhlenbeck random process, which is a Gaussian and zeroth-order Markovian process \cite[p. 85]{RASMUSSEN-2006}. We set \(\alpha = 0.1\) throughout the examples. However, in \S \ref{sec:optimal-parameters}, we will find the optimal set of covariance hyperparameters, including the optimal value for \(\alpha\).

To produce discrete data, we sample \(n\) points from \(\mathcal{D}\), which yields the vector of discretized data \(\vect{z}\), the design matrix \(\tens{\X}\), and the correlation matrix \(\tens{K}\). We set \(n = 50^2\) throughout the examples. However, in \S \ref{sec:scalability}, we examine the scalability of our method over various the number of points, \(n\).


\subsection{Maximizing Marginal Likelihood} \label{sec:max-LogL}

We aim to estimate \(\sigma\) and \(\sigma_0\) of the data (through \(\eta\)) by maximizing \(\ell_{\hat{\sigma}^2(\eta)}(\eta)\). The first derivative of \(\ell_{\hat{\sigma}^2(\eta)}(\eta)\) is shown by the solid black curve in \Cref{fig:LogLikelihoodDerivative} in the range \(\eta \in [10^{-3},10^3]\). The dashed and dash-dot black curves respectively represent the upper and lower bounds of the derivative of \(\ell_{\hat{\sigma}^2(\eta)}(\eta)\), which are given by \Cref{prop:der-bound}. The embedded frame demonstrates a portion of the diagram within \(4\) orders of magnitude smaller scale on the ordinate, where a root of the derivative is expected at \(\hat{\eta}\). The orange and green solid curves are respectively the first-order and second-order asymptotic approximations given in \Cref{prop:asymptote} (see also \Cref{rem:lower-asym}). They approximate the root respectively at \(\hat{\eta}_1 = 10^{1.98}\) and \(\hat{\eta}_2 = 10^{1.64}\). Although the asymptotes are expected to be valid only for \(\eta \gg \lambda_n = 10^{2.1}\), they nonetheless estimate the root remarkably close to the true value \(\hat{\eta} = 10^{1.30}\) (see \Cref{rem:valid-asym}). The asymptotes indicate there are no more roots at \(\eta > \mathcal{O}(\hat{\eta}_2)\), which, together with the smallest eigenvalue, \(\lambda_1 = 10^{-1.07}\), narrows the search for the roots to the interval \(\eta \in [\mathcal{O}(\lambda_1),\max(\mathcal{O}(\lambda_n),\mathcal{O}(\hat{\eta}_2))] \approx [10^{-2},10^{3}]\).

\begin{figure}[bt!]
    \centering
    \includegraphics[width=0.9\textwidth]{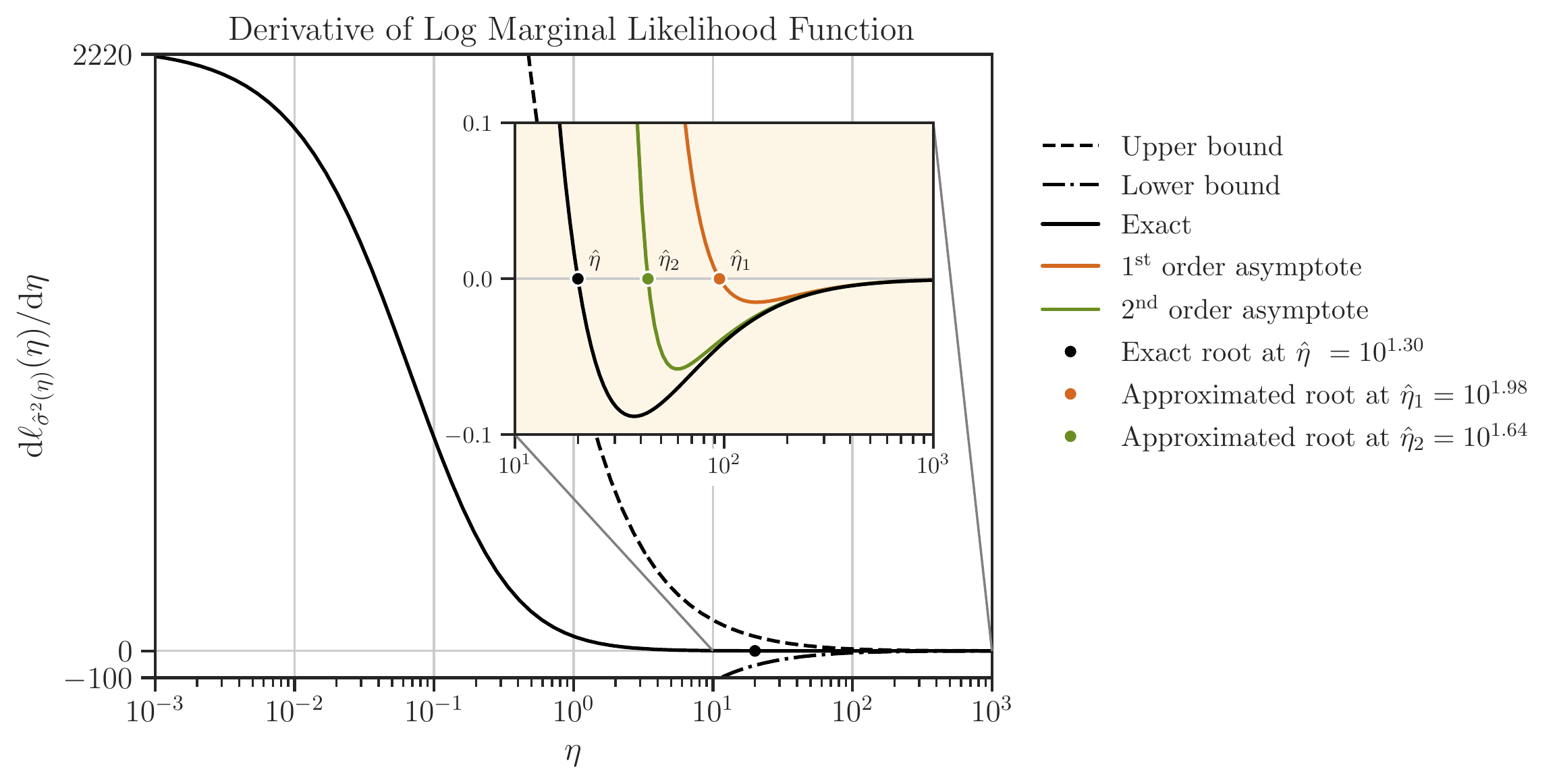}
    \caption{The first total derivative of the log marginal likelihood function \(\ell_{\hat{\sigma}^2(\eta)}(\eta)\). The embedded diagram illustrates the location of the function root in a significantly smaller scale. The colored curves are the asymptotes that approximate the black curve at large values of \(\eta\). The roots are shown by a dot on each curve.}
    \label{fig:LogLikelihoodDerivative}
\end{figure}

To find the root \(\hat{\eta}\) of \(\mathrm{d} \ell_{\hat{\sigma}^2(\eta)}(\eta) / \mathrm{d} \eta\), we opt for a Jacobian-free approach that does not require the second derivative, \(\mathrm{d}^2 \ell_{\hat{\sigma}^2(\eta)}(\eta) / \mathrm{d} \eta^2\). For our purpose, we use the method of \citet{CHANDRUPATLA-1997} which leverages both the robustness of the bisection method and the convergence rate of the inverse quadratic interpolation method. The performance of this method is comparable with the well-known Brent's root-finding algorithm. However, the Chandrupatla's method converges substantially faster on functions that are flat around their root. Further discussion and implementation details of this method can be found in \citet[\S 6.1.7.3]{SCHERER-2010}. Motivated by the multi-scale shape of the log marginal likelihood derivative in \Cref{fig:LogLikelihoodDerivative}, the Chandrupatla's method is desirable for our purpose. (However, we found that efficiency or accuracy was not considerably affected by the root-finding algorithm; \eg the root \(\hat{\eta} = 10^{1.30} \pm 10^{-6}\) in \Cref{fig:LogLikelihoodDerivative} was found with less than \(10\) iterations and \(15 \sim 20\) iterations respectively by the Chandrupatla and Brent methods.)

Based on the calculated \(\hat{\eta}\), we computed \(\hat{\sigma}(\hat{\eta}) = 0.0437\) from \Cref{thm:sigma}, which yields the estimate for the noise variance \(\hat{\sigma}_0 = 0.1958\). With \(2.09 \%\) relative error, the estimate \(\hat{\sigma}_0\) is fairly close to the standard deviation \(\sigma_0 = 0.2\) of the noise that we added to the data, which is a compelling result considering the high noise level compared to the mean function. 


\subsection{Comparison of Various Basis Functions and Noise Levels} \label{sec:various-noise}

We investigated the effect of various basis functions on noise estimation. \Cref{table:various-basis} shows results from using polynomial basis functions in \eqref{eq:poly-basis} with various orders \(q\), with the input noise \(\sigma_0\) is fixed. The last column shows the relative error of estimating the noise level, \ie \(|\sigma_0 - \hat{\sigma}_0| / \sigma_0\). 

Since the data with the mean function in \eqref{eq:mu-data} has a convex shape, the zeroth and first-order polynomial basis cannot fit the data properly, leading to \(13.8\%\) relative error in the estimation of \(\sigma_0\), which also leads to a non-negligible residual \(\delta\) with \(\sigma = 0.227\). In contrast, the estimation error by second and higher-order basis functions is significantly reduced. However, higher-order basis functions (\ie \(q > 6\)) become multi-collinear, and will generally introduce an undesirable over-fitting \cite[\S 3.3]{GELFAND-2010}.

\begin{table}[htb!]
    \centering
    \begin{small}
        \begin{tabular}{l c c c c c r}
        \toprule
        \multicolumn{3}{c}{Input parameters} & \multicolumn{4}{c}{Results} \\
        \cmidrule(lr){1-3} \cmidrule(lr){4-7}
        Basis function \(\vect{\basis}\) & \(m\) & \(\sigma_0\) & \(\log_{10}(\hat{\eta})\) & \(\hat{\sigma}\) & \(\hat{\sigma}_0\) & Error \\
        \midrule
        Polynomial, \(0\)\textsuperscript{th} order & \(1\)  & \(0.2\) & \(-0.2376\) & \(0.2268\) & \(0.1725\) & \(13.73 \%\) \\
        Polynomial, \(1\)\textsuperscript{st} order & \(3\)  & \(0.2\) & \(-0.2410\) & \(0.2275\) & \(0.1723\) & \(13.80 \%\) \\
        Polynomial, \(2\)\textsuperscript{nd} order & \(6\)  & \(0.2\) & \(+1.3007\) & \(0.0437\) & \(0.1958\) & \(2.09 \%\) \\
        Polynomial, \(3\)\textsuperscript{rd} order & \(10\) & \(0.2\) & \(+1.2527\) & \(0.0462\) & \(0.1955\) & \(2.20 \%\) \\
        Polynomial, \(4\)\textsuperscript{th} order & \(15\) & \(0.2\) & \(+2.4755\) & \(0.0114\) & \(0.1973\) & \(1.33 \%\) \\
        Polynomial, \(5\)\textsuperscript{th} order & \(21\) & \(0.2\) & \(+2.2975\) & \(0.0140\) & \(0.1972\) & \(1.36 \%\) \\
        Trigonometric                               & \(4\)  & \(0.2\) & \(\infty\)  & \(0.0000\) & \(0.1974\) & \(1.28 \%\) \\
        \bottomrule
    \end{tabular}
    \end{small}
    \caption{Comparison of \(\hat{\sigma}\) and \(\hat{\sigma}_0\) versus various basis functions.} \label{table:various-basis}
\end{table}

We note that by increasing the order of the polynomial basis functions, the estimation error of \(\hat{\sigma}_0\) reaches a limit and cannot be improved any further. This is due to the insufficiency of the number of samples \(n\). To verify that the persisting error is not due to the order of the basis functions, we can use  trigonometric basis functions instead, such as
\begin{equation}
    \vect{\basis}(\vect{x}) = \left( \sin(\pi x_1),\cos(\pi x_1),\sin(\pi x_2),\cos(\pi x_2)\right). \label{eq:trig-basis}
\end{equation}
The above basis is the same function type used in the mean of the data, \(\mu(\vect{x})\), in \eqref{eq:mu-data}. Hence, we expect the residual function \(\delta(\vect{x})\) vanish, since, in this case \(\vect{z} \in \range(\tens{\X})\) (see also \Cref{rem:z-rangeX}). As shown in the last row of the table for the trigonometric basis, \(\delta(\vect{x})\) vanishes as \(\hat{\sigma} = 0\). However, the estimate of data noise, \(\hat{\sigma}_0\), still shows \(1.28\%\) error. By increasing the data samples \(n\), the estimation error will be reduced, which we have verified.


\begin{figure}[bt!]
    \centering
    \includegraphics[width=0.9\textwidth]{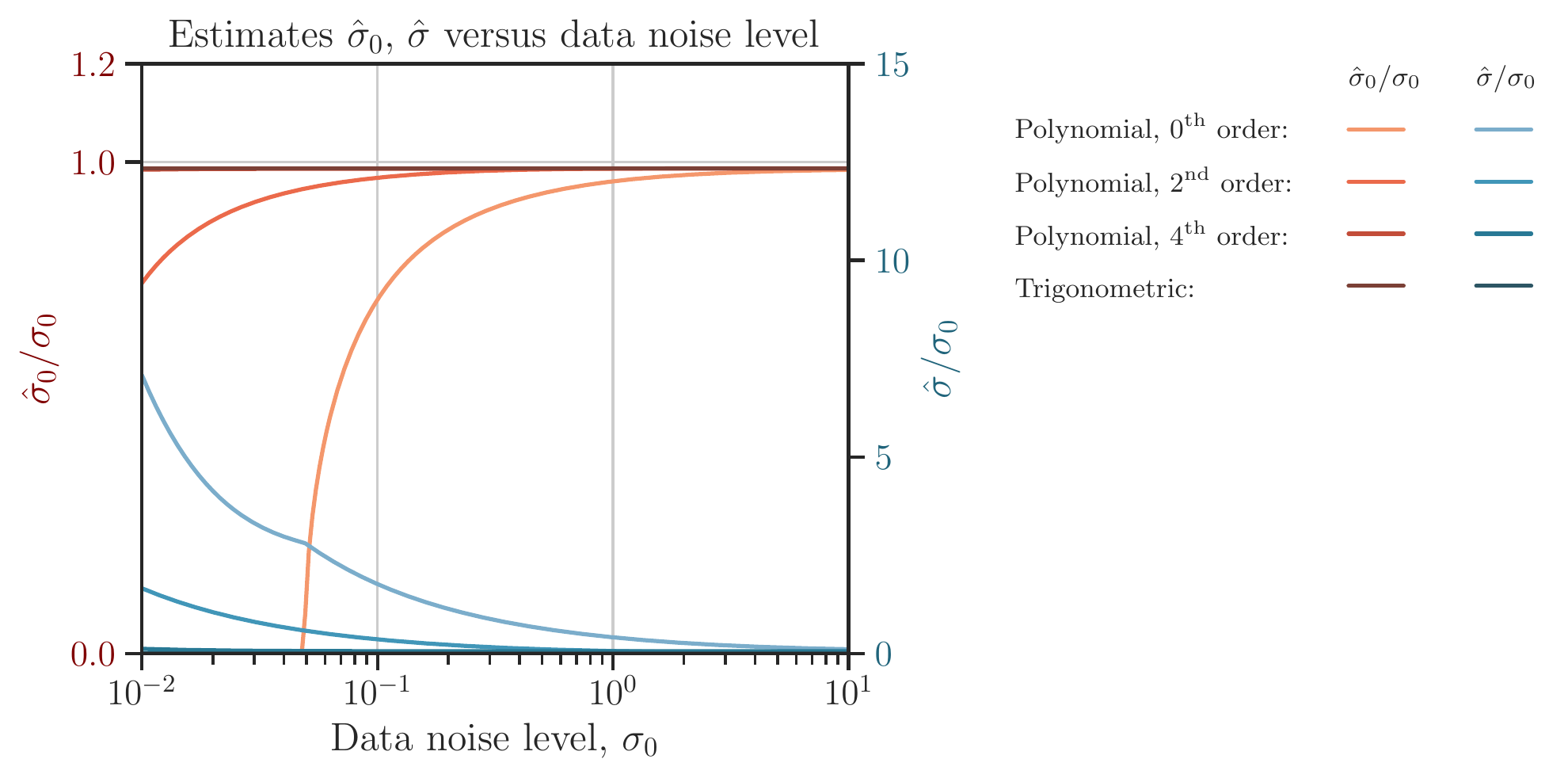}
    \caption{The ratio of estimated noise, \(\hat{\sigma}_0\), (left axis), and the estimated error variance, \(\hat{\sigma}\), (right axis) over the input data noise level, \(\sigma_0\), are compared with various polynomial basis functions. The curve \(\hat{\sigma}/\sigma_0\) for the fourth-order polynomial and trigonometric bases have almost vanished.}
    \label{fig:NoiseLevel}
\end{figure}

In \Cref{fig:NoiseLevel} the ratio of estimations of \(\hat{\sigma}_0\) (left axis), and \(\hat{\sigma}\) (right axis) over the data noise \(\sigma_0\), is shown versus the data noise level \(\sigma_0\) and various polynomial basis orders, \(q\). The fourth-order polynomial and the trigonometric bases produce almost identical results as the corresponding curves are overlaid. They can also estimate noise steadily for the entire range of the data noise level. We also note \(\hat{\sigma}_0 / \sigma_0\) for fourth-order and trigonometric bases is slightly less than \(1\), which the estimation can be improved by using more data samples. On the other hand, the second-order polynomial basis results in large estimation errors at lower noise levels, yet the estimation at \(\sigma_0 = 0.2\) has reasonable accuracy.


\subsection{Performance and Scalability} \label{sec:scalability}

We investigate the scalability of our method by the computational cost over the data size, \(n\). The numerical experiments were performed on an Intel Xeon E5-2640 v4 processor using shared memory parallelism. The computational costs are measured by the total CPU times of all computing cores. We consider numerical algorithms for both dense and sparse correlation matrices respectively in \S \ref{sec:dense} and \S \ref{sec:sparse}.


\subsubsection{Dense Correlation Matrix} \label{sec:dense}

The scalability of our method with dense correlation matrix \(\tens{K}\) is shown in \Cref{fig:ElapsedTime-Dense} by the curves using red themes. The dashed lines are the line-fit, which indicates a computational complexity of around \(\mathcal{O}(n^{2.5})\). The pre-computation (light red curve) is part of the computation before maximizing \(\ell\), which includes calculating \(\tau(\eta_i)\) at interpolant points \(\eta_i\) to interpolate the trace of \(\tens{K}_{\eta}^{-1}\) afterward based on \eqref{eq:approx} (see \S \ref{sec:pre-computation}). The cost of pre-computation is proportional to the number of interpolant points, \(\eta_i\). The interpolant points in our experiment are \(\eta_i \in \{1,10,40,10^2,10^3\}\) to cover a wide range of \(\eta\). The function \(\tau(\eta_i)\) at interpolant points is calculated by the Cholesky factorization method described in \eqref{eq:trace-chol}. The medium red curve in the figure shows the processing time of finding zeros of the derivative of \(\ell_{\hat{\sigma}^2(\eta)}(\eta)\) as described in \S \ref{sec:max-LogL} based on the numerical algorithm of \S \ref{sec:prelim-numerical}. The dark red curve is the total time of computation, which is the combination of the pre-computation and maximizing the likelihood function.

\begin{figure}[bt!]
    \centering
    \includegraphics[width=\textwidth]{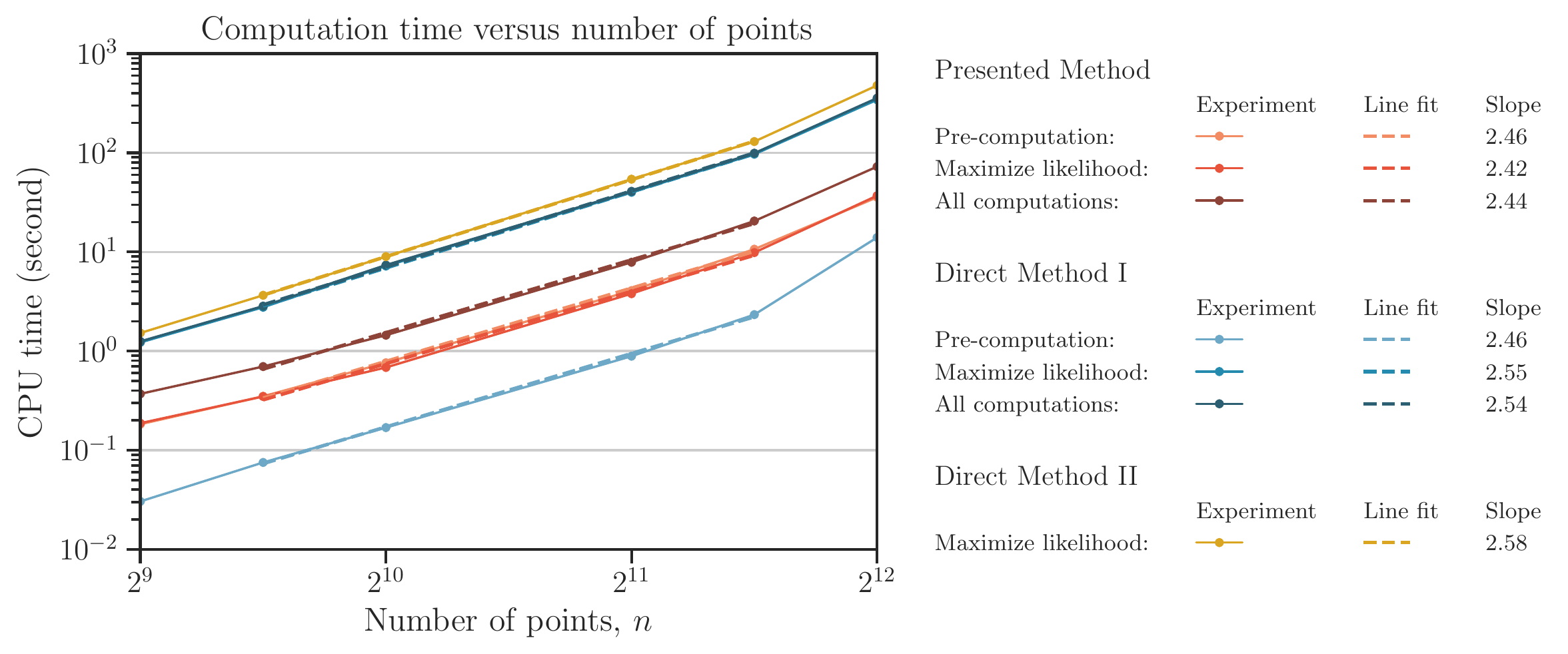}
    \caption{The elapsed time of computation versus data size for the dense correlation matrix.}
    \label{fig:ElapsedTime-Dense}
\end{figure}

We also compared the performance of our method with a conventional method of finding hyperparameters \((\sigma^2,\sigma_0^2)\). In the conventional approach, we maximize \(\ell(\sigma^2,\sigma_0^2)\) directly over the two-dimensional space of hyperparameters \((\sigma^2,\sigma_0^2)\) using an implementation of the Nelder-Mead optimization algorithm by \citet{GAO-2012}. In \Cref{fig:ElapsedTime-Dense}, this approach is indicated by the \emph{direct} method. The difference between the direct method I (blue theme curves) and direct method II (yellow curve) is in the computation of determinant term, \(\log |\gtens{\Sigma}|\), which is an expensive part of computing \(\ell\) using \eqref{eq:log-marginal-likelihood}. We describe both methods as follows.

In the first direct method, we obtain the log-determinant directly from the eigenvalues of \(\gtens{\Sigma}_{\sigma^2,\sigma_0^2} = \sigma^2 \tens{K} + \sigma_0^2 \tens{I}\) by,
\begin{equation*}
    \log |\gtens{\Sigma}_{\sigma^2,\sigma_0^2}| = \sum_{i=1}^n \log (\sigma^2 \lambda_i + \sigma_0^2),
\end{equation*}
where \(\lambda_i\) are the eigenvalues of \(\tens{K}\). For small and dense matrices (here, \(n \leq 2^{12}\)), all eigenvalues of \(\tens{K}\) can be efficiently obtained using tridiagonal reduction and shifted QR factorization. In \Cref{fig:ElapsedTime-Dense}, computing the log-determinant with eigenvalues is indicated as the pre-computation and shown by the light blue curve. The pre-computation of direct method I has almost an order of magnitude less processing time compared to our presented method, but with the similar scalability. However, the overall processing time of the direct method I (the dark blue curve) is an order of magnitude higher than our presented method. This is because our presented method, and the direct methods, reach convergence, respectively around \(10\) and \(400\) evaluations of the function \(\ell\). For all methods, the iterations are terminated when the convergence error of the estimated hyperparameters reached the tolerance of \(10^{-6}\).

For larger matrices, computing the log-determinant through its eigenvalues is not efficient. Instead, in the direct method II, we have computed the log-determinant of \(\gtens{\Sigma}_{\sigma^2,\eta} = \sigma^2 \tens{K}_{\eta}\) by the lower-triangular Cholesky factorization, \(\tens{L}_{\eta}\), of the matrix \(\tens{K}_{\eta}\) using,
\begin{equation}
    \log |\gtens{\Sigma}_{\sigma^2,\eta}| = n \log(\sigma^2) + 2 \trace(\log (\operatorname{diag}(\tens{L}_{\eta}))), \label{eq:log-det}
\end{equation}
where \(\operatorname{diag}(\tens{L}_{\eta})\) is the matrix of the diagonal elements of \(\tens{L}_{\eta}\). To obtain the above relation, we used the fact that the determinant of a lower-triangular matrix is the product of is diagonals.

The overall performance with the direct method II is shown by the yellow curve in \Cref{fig:ElapsedTime-Dense}. The direct method I performs slightly better than the direct method II, but, the latter has better scalability as we approach larger data sizes. As it can be seen from \Cref{fig:ElapsedTime-Dense}, the computational performance for larger dense matrices at \(n \geq 2^{12}\) deviates from their linear trend, which necessitates sparse matrix algorithms.


\subsubsection{Sparse Correlation Matrix} \label{sec:sparse}

A substantial difference in the computational cost between our presented method and the direct method is uncovered in large data sets. In the following, we compare these two methods for large data sets using a sparse correlation matrix.

To produce a sparse correlation matrix, a compactly supported kernel is used. Often, the tail of the correlation function is tapered to produce a compactly supported correlation kernel \citep{ZHANG-2008}, such as by multiplying the kernel \(K(\vect{x},\vect{x}')\) by the indicator function \(\mathbbm{1}_{K > \kappa}(\vect{x},\vect{x}')\), where \(\mathbbm{1}_{K > \kappa} = 1\) if \(K > \kappa\), and zero otherwise. The tapering threshold, \(\kappa\), should be large enough to reduce the density of the sparse matrix while carrying most of its information. On the other hand, \(\kappa\) should be small enough to not introduce undesirable oscillations in the spectral density of the kernel, which eradicates its positive-definiteness \citep{GENTON-2002}. In the following numerical experiment, we set the tapering threshold to \(\kappa = 0.03\) and the kernel decorrelation scale to \(\alpha = 0.005\). These settings result in a sparse correlation matrix with non-zero element density \(\rho = \operatorname{Vol}_d(-\alpha \log(\kappa)) \approx 10^{-3}\), which is the volume of a \(d\)-ball (here, \(d = 2\)) of radius \(-\alpha \log \kappa\).

We recall that the trace of \(\tens{K}_{\eta}^{-1}\) in the interpolation function \(\tau(\eta)\) is the most expensive term in the computation of the derivative of \(\ell\) and requires special attention. A possible approach is to use the sparse Cholesky factorization (such as by CHOLMOD \citet{CHEN-2008}) to compute the trace of \(\tens{K}_{\eta}^{-1}\) with the method described in \eqref{eq:trace-chol}. Instead, we employ a stochastic trace estimator as this is a more highly scalable class of methods. In particular, we used the stochastic Lanczos quadrature (see \S \ref{sec:pre-computation}) with Golub-Kahn-Lanczos bi-diagonalization technique described in \cite{UBARU-2017}. The computational cost of stochastic Lanczos quadrature is \(\mathcal{O}\left( (\operatorname{nnz}(\tens{K})l + l^2 n) n_v \right)\) (see \citet[p. 1083]{UBARU-2017}), where \(\operatorname{nnz}(\tens{K})\) is the number of non-zero elements of the sparse matrix \(\tens{K}\) and is equal to \(\rho n^2\). Also, \(l\) is the Lanczos degree which is the number of Lanczos iterations for Golub-Kahn-Lanczos bi-diagonalization, and \(n_v\) is the number of random vectors with Rademacher distribution for Monte-Carlo sampling. In our application, the cost of obtaining the interpolation function \(\tau(\eta)\) is \(p\) times the above-mentioned complexity, and we recall \(p\) is the number of interpolant points. Thus, the overall complexity of the calculation of \(\tau(\eta)\) is,
\begin{equation*}
    \mathcal{O} \left( (\rho n^2 + nl) n_v l p \right).
\end{equation*}
The processing time of the pre-computation (\ie calculation of \(\tau(\eta)\)) in our experiment with sparse matrices is shown by the light red curve in \Cref{fig:ElapsedTime-Sparse}. In the stochastic Lanczos quadrature method, we have employed \(n_v = 20\) number of Monte-Carlo random vectors with the Lanczos degree \(l = 20\) while keeping the sparse matrix density \(\rho\) constant throughout the experiment. We note that at \(n > 2^{15}\), the complexity of the method is proportional to \(\mathcal{O}(n^2)\), which can be verified in the figure. The medium red curve corresponds to the maximization of \(\ell\) using the method described in \ref{sec:prelim-numerical}. The dark red curve corresponds to the overall computation, which scales with the complexity of \(\mathcal{O}(n^2)\) as most of the computational cost is due to the pre-computation.

\begin{figure}[bt!]
    \centering
    \includegraphics[width=\textwidth]{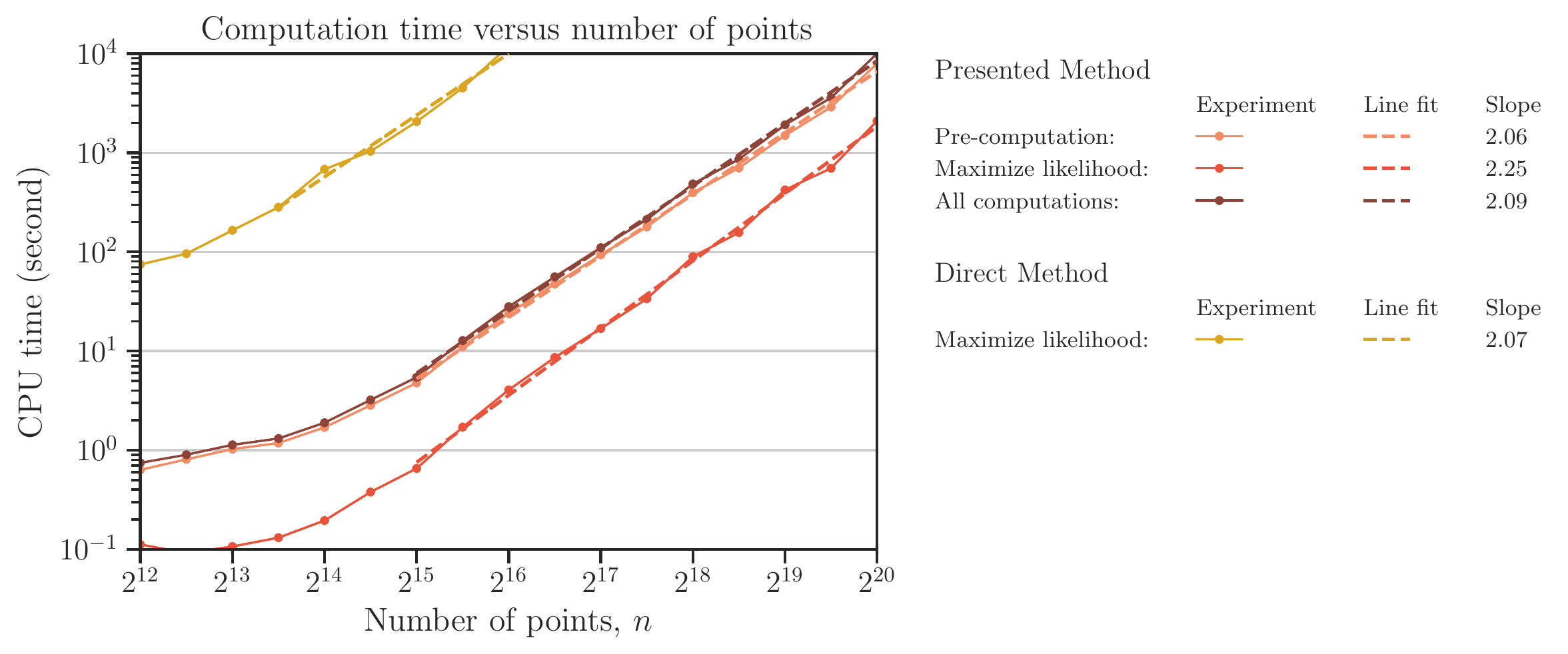}
    \caption{The elapsed time of computation versus data size for the sparse correlation matrix.}
    \label{fig:ElapsedTime-Sparse}
\end{figure}

We again compared the presented method with the direct method of optimizing \(\ell(\sigma^2,\sigma_0^2)\) in the hyperparameter space \((\sigma^2,\sigma_0^2)\). Recall that the computationally expensive term in the evaluation of \(\ell\) in \eqref{eq:log-marginal-likelihood} is the log-determinant, \(\log |\gtens{\Sigma}|\). Computing the determinant of large sparse matrices has been studied, such as by \cite{REUSKEN-2002}, \cite{IPSEN-2011}, and \cite{ERLEND-2014}. As before, we employ the stochastic Lanczos quadrature with Golub-Kahn-Lanczos bi-diagonalization technique to estimate log-determinant via trace by \eqref{eq:log-det}. For a fair comparison with our presented method, we set the same hyperparameters, namely, the number Monte-Carlo random vectors \(n_v = 20\) and the Lanczos degree \(l = 20\). The performance of the overall computation of the direct method is shown by the yellow curve in \Cref{fig:ElapsedTime-Sparse}. The direct method demonstrates similar scalability, \ie \(\mathcal{O}(n^2)\), but it takes two to three orders of magnitude more processing time than our presented method.

It is important to note that besides the performance advantage, the presented method is robust in the sense that it converges for all initial guess points in the root-finding algorithm. In contrast, the direct method is very sensitive to the initial guess of hyperparameters; often the solution diverges or converges to a wrong local maximum of \(\ell\) instead of the global maximum. We will compare the robustness of the methods with examples in \S \ref{sec:optimal-parameters}.


\subsection{Modeling with Mat\'{e}rn Class} \label{sec:optimal-parameters}

So far, we have found optimal values of the covariance function in the hyperparameter space \(\vect{\theta} = (\sigma^2,\sigma_0^2)\) using the exponential decay correlation function in \eqref{eq:exp-decay}, assuming the decorrelation scale hyperparameter \(\alpha\) is given. In the following, we relax the exact nature of the correlation function. To this end, we employ the Mat\'{e}rn correlation described in \S \ref{sec:matern}, which is a richer class of functions with an additional hyperparameter. Moreover, we allow the hyperparameters of the correlation to be found optimally by maximizing the posterior distribution of all hyperparameters. We study two cases, the posterior with and without prior distributions respectively in \S \ref{sec:without-prior} and \S \ref{sec:with-prior}.


\subsubsection{Mat\'{e}rn Correlation Function} \label{sec:matern}

The isotropic correlation function of \cite{MATERN-1960} (see also \cite[p. 31]{STEIN-1999}) is given by
\begin{equation} \label{eq:matern-corr}
    K(\vect{x},\vect{x}'|\alpha,\nu) = \frac{2^{1-\nu}}{\Gamma(\nu)} \left( \sqrt{2 \nu} \frac{\| \vect{x} - \vect{x}' \|_2}{\alpha} \right)^{\nu} K_{\nu} \left( \sqrt{2 \nu} \frac{\| \vect{x} - \vect{x}' \|_2}{\alpha} \right),
\end{equation}
where \(\Gamma\) is the Gamma function and \(K_{\nu}\) is the modified Bessel function of the second kind of order \(\nu\) \citep[\S 9.6]{ABRAMOWITZ-1964}. The hyperparameter \(\alpha > 0\) is the decorrelation scale of the kernel as before. The Mat\'{e}rn class is a powerful formulation because the hyperparameter \(\nu > 0\) can enable modulation of the smoothness of the underlying random process. A Gaussian process with such correlation function is \(\lceil v \rceil -1\) differentiable in \(L^2\) (square mean) sense, where \(\lceil \cdot \rceil\) is the integer ceiling function.

The Mat\'{e}rn kernel is a widely accepted correlation function for both univariate applications \citep{GUTTORP-2006} as well as multivariate applications \citep{GNEITING-2010}. Some known correlation functions can be obtained from the Mat\'{e}rn function for various \(\nu\) \cite[Table 1]{GUTTORP-2006}. In the limit \(\nu \to 0^{+}\), the correlation function becomes discontinuous and the correlation matrix \(\tens{K}\) tends to the identity matrix, \(\tens{I}\), which then, \(\sigma\) and \(\sigma_0\) become indistinguishable. 
At \(\nu = \frac{1}{2}\), the Mat\'{e}rn correlation function reduces to the exponential decay kernel that we used in \eqref{eq:exp-decay}. 
The values \(\nu = \frac{3}{2}\) and \(\frac{5}{2}\) correspond to second- and third-order auto-regressive models, and they are commonly used in machine learning applications \cite[p. 85]{RASMUSSEN-2006}. Also, in the limit \(\nu \to \infty\), the Mat\'{e}rn correlation function approaches the smooth Gaussian kernel,
\begin{equation}
    K(\vect{x},\vect{x}'|\alpha,\infty) = \exp \left( -\frac{1}{2} \frac{\| \vect{x} - \vect{x}' \|_2^2}{\alpha^2}\right). \label{eq:gaussian-kernel}
\end{equation}
In practice, the Mat\'{e}rn function with \(\nu > 25\) is almost Gaussian with less than \(1 \%\) difference error. The Mat\'{e}rn function for various smoothness hyperparameters \(\nu\) is shown in \Cref{fig:MaternKernel}.

\begin{figure}[bt!]
    \centering
    \includegraphics[width=0.48\textwidth]{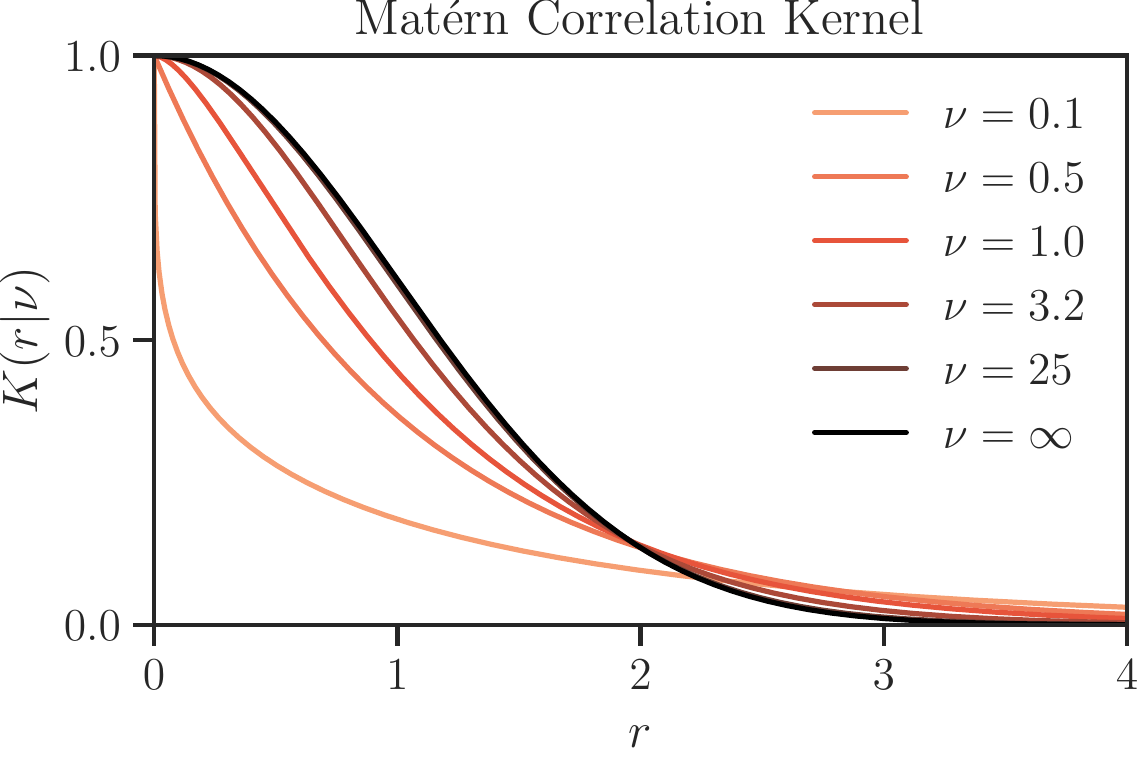}
    \caption{Mat\'{e}rn function for the normalized isotropic radial distance \(r \coloneqq \| \vect{x} - \vect{x}' \|_2 / \alpha\). The curve with \(\nu = 25\) is overlaid by the Gaussian kernel shown by the black curve.}
    \label{fig:MaternKernel}
\end{figure}

For the examples in the previous sections, we sought \(\hat{\vect{\theta}} = (\hat{\sigma}^2,\hat{\sigma}_0^2)\) assuming fixed structure hyperparameters \((\alpha,\nu) = (0.1,0.5)\). In the following, we aim to find an optimal value of the broader hyperparameters, \(\hat{\vect{\theta}} = (\hat{\alpha},\hat{\nu},\hat{\sigma}^2,\hat{\sigma}_0^2)\), for the covariance matrix,
\begin{equation*}
    \tens{\Sigma}_{\sigma^2,\sigma_0^2,\alpha,\nu} = \sigma^2 \tens{K}_{\alpha,\nu} + \sigma_0^2 \tens{I}.
\end{equation*}


\subsubsection{Optimal Hyperparameters Using Uniform Priors} \label{sec:without-prior}

For a given set of hyperparameters \((\alpha,\nu)\), we can find the optimal hyperparameters \((\hat{\sigma}^2,\hat{\sigma}_0^2)\) using the presented method. Thus, instead of maximizing the log marginal likelihood \(\ell(\alpha,\nu,\sigma^2,\sigma_0^2)\) over the entire four-dimensional space of hyperparameters, we only maximize the \emph{profile} log marginal likelihood function \(\ell_{\hat{\sigma}^2(\hat{\eta}),\hat{\eta}}(\alpha,\nu)\) in a two-dimensional space, where the other two hyperparameters, \((\sigma^2,\sigma_0^2)\), are profiled out using their optimal values \(\hat{\sigma}^2(\alpha,\nu)\) and \(\hat{\sigma}_0^2(\alpha,\nu)\), or equivalently, through \(\hat{\eta}(\alpha,\nu)\). \Cref{fig:Profile-Log-Likelihood} illustrates the difference of profiled log marginal likelihood function \(\ell_{\hat{\sigma}^2(\hat{\eta}),\hat{\eta}}(\alpha,\nu)\) with its maximum. The maximum of the plot, \(\ell_{\hat{\sigma}^2(\hat{\eta}),\hat{\eta}}(\hat{\alpha},\hat{\nu})\), at the location \((\hat{\alpha},\hat{\nu})\) is shown by the white dot on the top edge of the diagram. This function asymptotically approaches to its maximal point at \(\nu \to \infty\). Hence, we limited the domain to \(\nu < 25\) as the variation above \(\nu > 25\) is insignificant (see also \Cref{fig:MaternKernel}). In this figure, we have used a data size of \(n = 30^2\), an additive noise with standard deviation \(\sigma_0 = 0.2\), and the quadratic polynomial basis functions, \ie \(q = 2\).

\begin{figure}[bt!]
    \centering
    \includegraphics[width=0.6\textwidth]{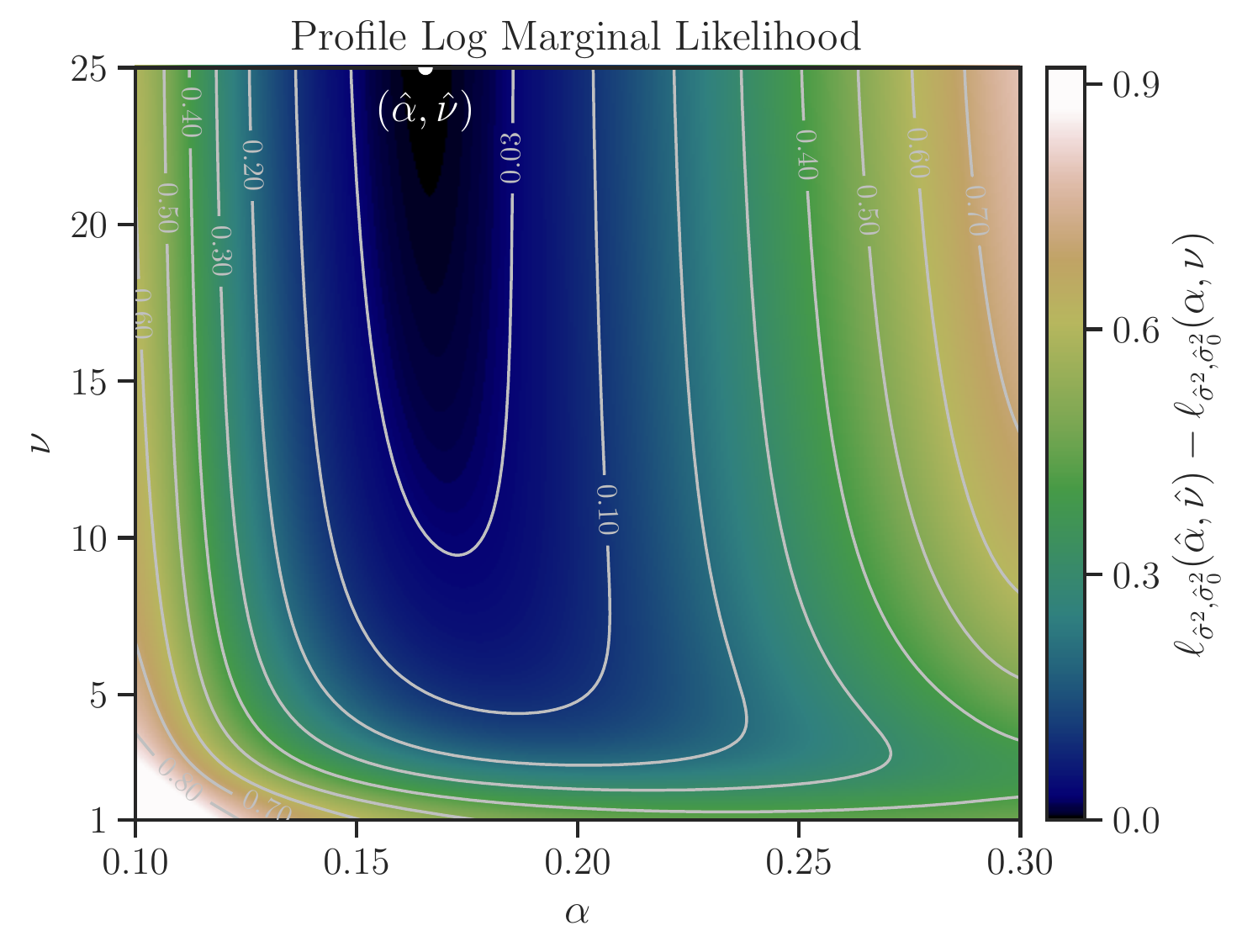}
    \caption{The log marginal likelihood \(\ell\) as a function of hyperparameters \((\alpha,\nu)\), where the other two hyperparameters, \ie \((\sigma^2,\sigma_0^2)\), are profiled out using their optimal estimates \(\hat{\sigma}^2(\alpha,\nu)\) and \(\hat{\sigma}_0^2(\alpha,\nu)\). The plot shows the difference of \(\ell\) with its maximum value at the point \((\hat{\alpha},\hat{\nu})\).}
    \label{fig:Profile-Log-Likelihood}
\end{figure}

Note that in the preceding analysis, we used \((\alpha,\nu) = (0.1,0.5)\), which is not optimal as revealed by \Cref{fig:Profile-Log-Likelihood}. To find the maximal point, we used Nelder-Mead's simplicial algorithm \citep{GAO-2012}, which is a direct optimization method that does not require evaluating the Jacobian or Hessian of the function. Between each iteration of the optimization algorithm, in an intermediate step, we find \((\hat{\sigma}^2,\hat{\sigma}_0^2)\) using our method. The algorithm is terminated when all estimated hyperparameters reach an accuracy of \(10^{-4}\) tolerance after \(N\) evaluations of the function \(\ell\). The resulting optimal hyperparameters are shown in the first row of \Cref{table:without-prior}. The estimated standard deviation of the noise in the seventh column of the table is very close to the actual data noise, \ie \(\sigma_0 = 0.2\).

\begin{table}[bht!]
    \centering
    \begin{adjustbox}{width=1\textwidth}
    \begin{small}
    \begin{tabular}{l l l c c c c c c}
        \toprule
        \multicolumn{3}{c}{Optimization Method} & \multicolumn{4}{c}{Optimal Hyperparameters (\(\pm 10^{-4}\))} & \multicolumn{2}{c}{Convergence} \\
        \cmidrule(lr){1-3} \cmidrule(lr){4-7} \cmidrule(lr){8-9}
        Param. Space & Initial Guess & Algorithm & \(\hat{\alpha}\) & \(\hat{\nu}\) & \(\hat{\sigma}\) & \(\hat{\sigma}_0\) & \(\max(\ell)\) & \(N\) \\
        \midrule
        \((\alpha,\nu)\) using \(\eta\) & \((0.1,1)\) & Nelder-Mead & \(0.1652\) & \(24.9999\) & \(0.0495\) & \(0.2031\) & \(958.5051\) & \(142\) \\
        \((\alpha,\nu,\sigma,\sigma_0)\) & \((0.1,1,0.05,0.05)\) & Nelder-Mead & \(0.1652\) & \(24.9997\) & \(0.0495\) & \(0.2031\) & \(958.5051\) & \(938\) \\
        \((\alpha,\nu,\sigma,\sigma_0)\) & \((0.1,1,0.02,0.02)\) & Nelder-Mead & \(0.0000\) & \(00.0000\) & \(0.0000\) & \(0.1172\) & \(527.4998\) & \(507\) \\
        \((\alpha,\nu,\sigma,\sigma_0)\) & Latin-Hypercube & Diff. Evolut. & \(0.1656\) & \(24.4360\) & \(0.0495\) & \(0.2031\) & \(958.5047\) & \(13640\) \\
    \bottomrule
    \end{tabular}
    \end{small}
    \end{adjustbox}
    \caption{Comparison of approaches to find the optimal hyperparameters of the Mat\'{e}rn covariance function using uniform prior distributions. The first row is based on our method, and the rest of the rows use the direct method. In the second and third rows, local optimization is applied using different initial guesses. In the fourth row, global optimization is used.} \label{table:without-prior}
\end{table}

We note that our presented method is robust in the sense that regardless of the initial guesses (\eg \((\alpha,\nu) = (0.1,1)\) given in the second column of the table), the algorithm converges to the unique maximal point with a relatively small number of function evaluations, \eg \(N \sim 100\). To better observe its robustness, we compare our method with the conventional method of finding the optimal hyperparameters by maximizing \(\ell\) directly in the four-dimensional space of \((\alpha,\nu,\sigma^2,\sigma_0^2)\).

Unfortunately, the solution with the direct approach using local optimization is prone to divergence. We can attempt to avoid divergence by imposing bounds on the hyperparameters using uniform prior distributions in a finite domain as
\begin{equation*}
    p(\nu) = H(\nu) - H(\nu - 25),
    \qquad \text{and} \qquad
    p(\alpha) = H(\alpha),
\end{equation*}
where \(H\) is the Heaviside step function. We note, even with the above priors that restrict the domain, the solution of the direct method with a reasonable initial guess of hyperparameters is not guaranteed to converge, and those solutions that converge, only do so because of the prior constraint. With the above priors, we maximize the profiled log posterior distribution
\begin{equation}
    \log p_{\hat{\sigma}^2(\hat{\eta}),\hat{\eta}}(\alpha,\nu | \vect{z}) = \ell_{\hat{\sigma}^2(\hat{\eta}),\hat{\eta}}(\alpha,\nu) + \log p(\alpha) + \log p(\nu). \label{eq:post}
\end{equation}

Examples of the direct optimization of the above posterior are given in the second to the fourth rows of the table \Cref{table:without-prior}. In the second row, we set the initial value \(\sigma = 0.05\), which is relatively close to the solution \(\hat{\sigma} = 0.0495\). However, it takes a significantly large number of iterations to achieve convergence of the solution despite the fortuitous initial condition. We note that without the constraining prior distributions, this solution would not converge. In the third row of the table, we demonstrated the sensitivity of the initial guess for the direct method, where, by changing the initial \(\sigma\) and \(\sigma_0\) slightly compared to the second row, the solution diverges. In this example, due to the implication of the bounding prior distributions, the diverged solution is constrained to the corner of the domain.

We observe that the profile log marginal likelihood \(\ell_{\hat{\sigma}^2(\hat{\eta}),\hat{\eta}}(\alpha,\nu)\) in our example appears to have a single maximum point. Therefore, local search optimization, such as the Nelder-Mead method, can efficiently locate its maximum. In contrast, a local search method is less exploitable to locate the global maxima of the log marginal likelihood \(\ell(\alpha,\nu,\sigma^2,\sigma_0^2)\) due to the potential of being trapped in a local minimum, which necessitates the utilization of a global search method. To this end, we applied the differential evolution optimization method \citep{STORN-1997} with best/1/exp strategy and 400 initial guess points distributed by the Latin hypercube method. The results are given in the fourth row of \Cref{table:without-prior}. This method finds the maximum point, but at a significant cost of many function evaluations.


\subsubsection{Optimal Hyperparameters Using Nonuniform Priors} \label{sec:with-prior}

In the previous section, we estimated the smoothness hyperparameter \(\hat{\nu} = 25\) (or \(\hat{\nu} = \infty\), since we imposed a constraint). This suggests the data in our example is best represented by a Gaussian correlation function\footnote{Note that the Gaussian function here (as a spatial correlation) is unrelated to the assumption of the Gaussian process prior for \(z(\vect{x})\).}, \eqref{eq:gaussian-kernel}, which has a smooth sample path. However, realizations of natural phenomena are rarely represented by smooth stochastic processes, particularly, when observed by insufficient samples. To moderate the smoothness hyperparameter, \(\nu\), we apply nonuniform prior distributions. Using the non-informative priors given by \cite{HANDCOCK-1993,HANDCOCK-1994}, we set
\begin{equation}
    p(\nu) = \frac{H(\nu)}{\left( 1 + \frac{\nu}{25} \right)^2}, 
    \qquad \text{and} \qquad
    p(\alpha) = \frac{H(\alpha)}{\left( 1 + \alpha \right)^2}. \label{eq:inv-square-prior}
\end{equation}
The rationale for the prior \(p(\alpha)\) in the above is to limit the decorrelation scale since the data points in our example are confined in the unit square \(\mathcal{D}\). Besides the above distributions, many other priors are also commonly used, such as Gaussian-gamma conjugate.

\begin{remark}[Priors for \(\sigma^2\) and \(\sigma_0^2\)]
    Prior distributions may also be chosen for other hyperparameters, \ie \((\sigma^2,\sigma_0^2)\). For instance, \cite{BERGER-2001} and \cite{PAULO-2005} investigated various priors for spatial data and Gaussian processes, such as Jeffery's reference prior, Jefferey's rule prior, and independence Jefferey's prior. To use a nonuniform prior for either of \(\sigma^2\) or \(\sigma_0^2\), our formulations respectively in \Cref{thm:sigma} and \Cref{thm:eta} should be modified slightly to incorporate the priors. For simplicity, we use uniform priors for \(\sigma^2\) and \(\sigma_0^2\).
\end{remark}

The profile marginal posterior \(\log p_{\hat{\sigma}^2(\eta),\hat{\eta}}(\alpha,\nu)\) can be obtained from the profile marginal likelihood and the priors \eqref{eq:inv-square-prior} using \eqref{eq:post}. The profile marginal posterior is shown in \Cref{fig:Profile-Log-Posterior} using the same data and model as before. Due to the presence of nonuniform priors, the optimal point, \((\hat{\alpha},\hat{\nu})\), on the figure is relocated to lower values of \(\nu\) as compared to \Cref{fig:Profile-Log-Likelihood} without priors. However, \(\hat{\alpha}\) remains relatively the same.

\begin{figure}[bt!]
    \centering
    \includegraphics[width=0.6\textwidth]{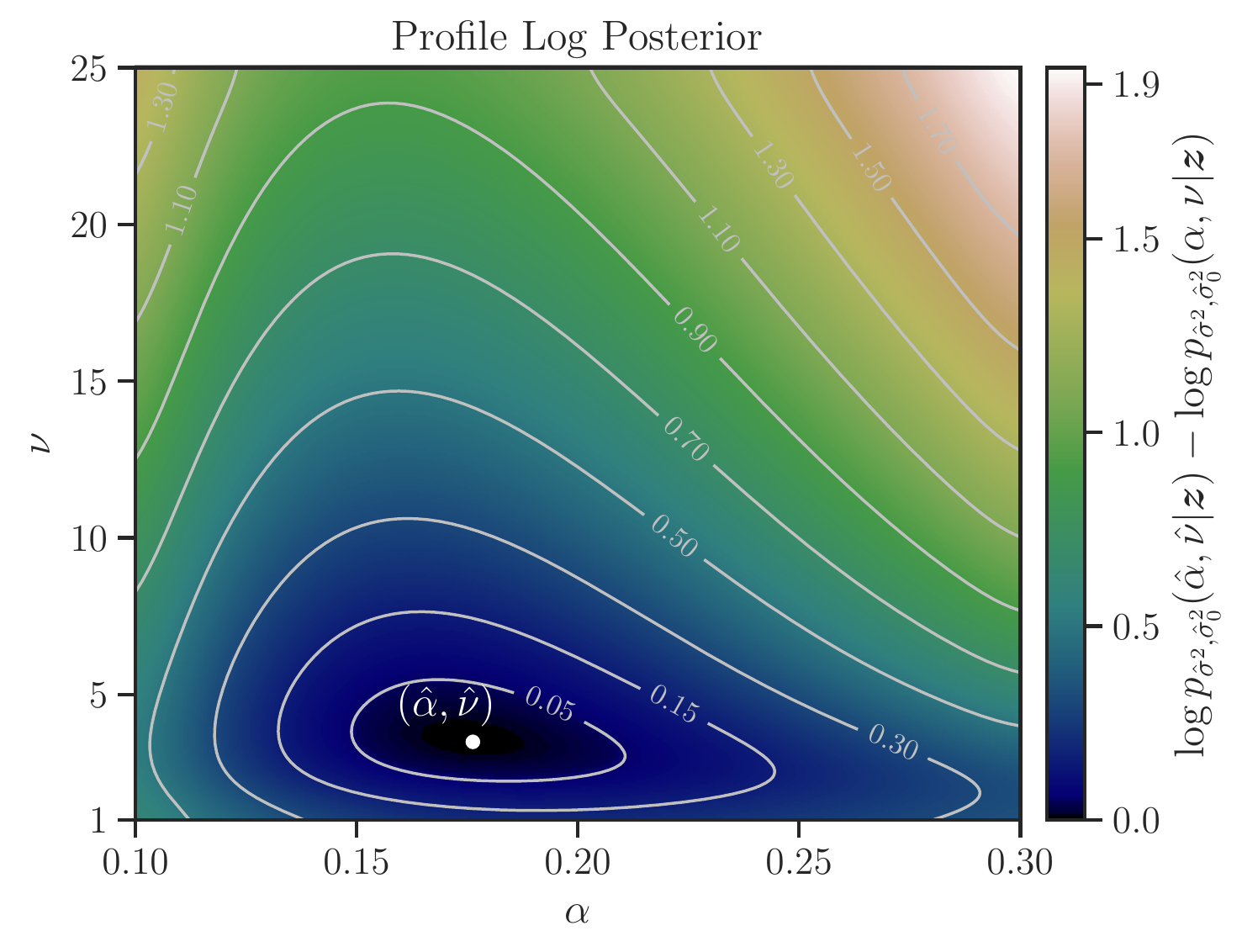}
    \caption{The marginal posterior, \(\log p(\alpha,\nu | \vect{z})\), using the inverse squared priors for hyperparameters \((\alpha,\nu)\). The hyperparameters \((\sigma^2,\sigma_0^2)\) are profiled out by using their optimal estimates \(\hat{\sigma}^2(\alpha,\nu)\) and \(\hat{\sigma}_0^2(\alpha,\nu)\). The plot shows the difference between \(\log p(\alpha,\nu | \vect{z})\) with its maximum at the point \((\hat{\alpha},\hat{\nu})\).}
    \label{fig:Profile-Log-Posterior}
\end{figure}

To find the optimal values of the hyperparameters, we use the Nelder-Mead local search optimization on the profile marginal posterior. The results, using an accuracy of \(10^{-4}\), are given in the first row of \Cref{table:with-prior}. The estimates \(\hat{\sigma}\) and \(\hat{\sigma}_0\) using the priors are identical to the results of \Cref{table:without-prior} without using the priors. However, \(\hat{\nu}\) is moderated considerably, indicating the correlation function is third-order differentiable. 

\begin{table}[bht!]
    \centering
    \begin{adjustbox}{width=1\textwidth}
    \begin{small}
    \begin{tabular}{l l l c c c c c c}
        \toprule
        \multicolumn{3}{c}{Optimization Method} & \multicolumn{4}{c}{Optimal Hyperparameters (\(\pm 10^{-4}\))} & \multicolumn{2}{c}{Convergence} \\
        \cmidrule(lr){1-3} \cmidrule(lr){4-7} \cmidrule(lr){8-9}
        Param. Space & Initial Guess & Algorithm & \(\hat{\alpha}\) & \(\hat{\nu}\) & \(\hat{\sigma}\) & \(\hat{\sigma}_0\) & \(\operatorname{max}(p)\) & \(N\) \\
        \midrule
        \((\alpha,\nu)\) using \(\eta\) & \((0.1,1)\) & Nelder-Mead & \(0.1769\) & \(3.2098\) & \(0.0495\) & \(0.2031\) & \(957.7796\) & \(82\) \\
        \((\alpha,\nu,\sigma,\sigma_0)\) & \((0.1,1,0.05,0.05)\) & Nelder-Mead & \(0.0235\) & \(0.0545\) & \(0.0416\) & \(0.2012\) & \(955.5810\) & \(444\) \\
        \((\alpha,\nu,\sigma,\sigma_0)\) & Latin-Hypercube & Diff. Evolut. & \(0.1769\) & \(3.2037\) & \(0.0498\) & \(0.2030\) & \(957.7796\) & \(8065\) \\
    \bottomrule
    \end{tabular}
    \end{small}
    \end{adjustbox}
    \caption{Comparison of approaches to find the optimal hyperparameters of the Mat\'{e}rn covariance function using nonuniform prior distributions. The first row is based on our method, and the rest of the rows use the direct method. In the second row, local optimization is applied, while in the third row, global optimization is used.} \label{table:with-prior}
\end{table}

We compared our method with the direct optimization of the posterior on the four-dimensional space of all hyperparameters, \((\alpha,\nu,\sigma^2,\sigma_0^2)\), with the results shown in the second and third rows of the \Cref{table:with-prior}. In this approach, the Nelder-Mead algorithm fails to converge to the global maximum for almost all initial guess points we tried. For example, in the second row of the table, we have examined the same initial point as in \Cref{table:without-prior}, but the solution converges to a local maximum with a lower posterior value than the global maximum. In the third row of the table, we have shown the optimization with the differential evolution algorithm with best/1/bin strategy and 400 initial guess points distributed by the Latin hypercube method. As a global search method, the differential evolution optimization successfully finds the global maximum, but again, with a significant cost of many function evaluations.


\section{Summary and Conclusion} \label{sec:conclusion}

In this paper, we have presented an efficient method for estimating the covariance hyperparameters of linear models for Gaussian process regression. The model we studied features correlated residual error, and uncorrelated additive noise. Namely, the presented method can efficiently estimate the variance of the residual error, \(\sigma^2\), and the variance of the additive noise, \(\sigma_0^2\), based on maximizing the marginal likelihood function. We summarize the paper as follows.

\begin{enumerate}[leftmargin=*,align=left]
    \item We explicitly derived the derivatives of the marginal likelihood function with respect to an arbitrary hyperparameter of the covariance function in \Cref{prop:der-L}. By solving for the roots of the derivative with respect to the arbitrary hyperparameter, we obtain its optimal value. The hyperparameters of interest herein were the variances \((\sigma^2,\sigma_0^2)\). 

    \item We reduced the dimensionality of the hyperparameter estimation. Specifically, in \Cref{prop:der-eta}, we obtained the derivatives of the marginal likelihood function with respect to the ratio of the noise variance over the residual error variance, as a new hyperparameter \(\eta\). This formulation enables simplification to a function of only one hyperparameter. Consequently, the estimation of this hyperparameter is reduced to a univariate root-finding problem as presented in \Cref{thm:eta}. The other variance hyperparameters are then retrieved by \Cref{thm:sigma}.

    \item We investigated the properties of the likelihood function. 
    We obtained the bounds of the marginal likelihood function and its derivatives in \Cref{prop:der-bound} and \Cref{cor:l-bound}. These bounds depend on the largest and the smallest eigenvalues of the correlation matrix. The bounds may be used to define an interval to initially search for the roots of the derivative of the marginal likelihood function. 

    \item We derived an asymptotic approximation of the derivative of the marginal likelihood function in \Cref{prop:asymptote}. Unlike the likelihood function, evaluating the asymptotic relation is inexpensive, and can be employed to obtain further knowledge of the problem before the hyperparameter estimation. For instance, the asymptotic relation can roughly estimate the location of the root of the derivative at large values of \(\eta\).
\end{enumerate}

The presented method was extensively examined on variations of a test problem, for which the results can be readily reproduced. We highlight the results of our numerical experiments as follows.

\begin{enumerate}[leftmargin=*,align=left]
    \item Our numerical examples demonstrated that the derived asymptotic relations approximate the root quite close to its true value, with the second-order asymptote providing a better result. Also, the asymptotic approximations may specify a rough interval in which to locate the roots.

    \item We compared our presented method with traditional (direct) hyperparameter estimation methods, where the marginal likelihood function is maximized directly in the two-dimensional space of hyperparameters. While only reducing one dimension with our method, the implications on performance gain are significant for both small and large data sizes. (a) For small data and dense correlation matrix, we observed an \(\mathcal{O}(n^{2.5})\) complexity and around an order of magnitude performance gain compared to the direct method. (b) For large data and sparse correlation matrix, the computational complexity in our numerical experiment steadily scales by \(\mathcal{O}(n^{2})\) with the performance gain up to three orders of magnitudes compared to the direct method.

    \item In a more general experiment, we relaxed assumptions regarding the nature of the spatial correlation of the data. This led to the estimation of hyperparameters of a general Mat\'{e}rn covariance function (described by the decorrelation \(\alpha\) and  smoothness \(\nu\)) in addition to the error variance and noise variance. We compared the direct hyperparameter estimation problem for these four variables \((\alpha,\nu,\sigma^2,\sigma_0^2)\) against our presented method with three concomitant variables \((\alpha,\nu,\eta)\). While having only one less hyperparameter, our presented method demonstrated a considerable advantage. Namely, our method was far more robust and insensitive to the initial guess of hyperparameters. In contrast, a direct optimization with a local search algorithm encountered many difficulties, such as sensitivity to the initial guess of hyperparameters and divergence of the solution. 
\end{enumerate}

\paragraph{Acknowledgments.}
The authors acknowledge support from the National Science Foundation, award number 1520825, and American Heart Association, award number 18EIA33900046.



\begin{appendices}
\section{Proofs of \texorpdfstring{\Cref{sec:properties}}{Section 4}}


\subsection{Proof of \texorpdfstring{\Cref{prop:der-bound}}{Proposition 8}} \label{prf:der-bound}

    We prove \eqref{eq:ldot-bound} and \eqref{eq:lddot-bound} respectively from the \ref{item:lambda1} to \ref{item:lambda5} and from \ref{item:lambda6} to \ref{item:lambda8} in the following.

    \begin{enumerate}[leftmargin=*,align=left,label*=\emph{Step (\roman*).},ref=step (\roman*),wide]
        \item\label{item:lambda1} By substituting \(\hat{\sigma}^2(\eta)\) from \Cref{thm:sigma} into the first derivative of \(\ell\) in \Cref{prop:der-eta}, we have,
    \begin{equation}
        \frac{\mathrm{d} \ell_{\hat{\sigma}^2(\eta)}(\eta)}{\mathrm{d} \eta} = -\frac{n-m}{2} \left( \frac{\trace(\tens{M}_{1,\eta})}{n-m} - \frac{\vect{z}^{\intercal} \tens{M}_{1,\eta}^2 \vect{z}}{\vect{z}^{\intercal} \tens{M}_{1,\eta} \vect{z}} \right). \label{eq:der-L-alt-3}
    \end{equation}
        We find bounds for the two terms in the parenthesis in the right-side of the above as follows.

        \item\label{item:lambda2} Let \(\eigM_i\) again denote the eigenvalues of \(\tens{M}_{1,\eta}\). Recall from \ref{item:gh1} in the proof of \Cref{prop:sign-indefinite} that \(\eigM_1 = \dots = \eigM_m = 0\), while the rest of eigenvalues \(\eigM_i\), \(i > m\), are positive. Also, from \eqref{eq:mu-bar-1} and \eqref{eq:mu-bar-ineq} recall that,
    \begin{equation}
        \eigM_{m+1} \leq \frac{\trace(\tens{M}_{1,\eta})}{n-m} \leq \eigM_n. \label{eq:mean-mu}
    \end{equation}

    \item\label{item:lambda3} \(\tens{M}_{1,\eta}\) is positive semi-definite, so \(\tens{M}_{1,\eta}^{\sfrac{1}{2}}\) is well-defined. Let \(\vect{w} \coloneqq \tens{M}_{1,\eta}^{\sfrac{1}{2}} \vect{z}\). Thus, the second term in the right-side of \eqref{eq:der-L-alt-3} is,
    \begin{equation*}
        \frac{\vect{z}^{\intercal} \tens{M}_{1,\eta}^2 \vect{z}}{\vect{z}^{\intercal} \tens{M}_{1,\eta} \vect{z}} = \frac{\vect{w}^{\intercal} \tens{M}_{1,\eta} \vect{w}}{\| \vect{w} \|^2},
    \end{equation*}
    which is the Rayleigh quotient for the matrix \(\tens{M}_{1,\eta}\). By the Courant-Fisher's mini-max principle \cite[Theorem 4.2.11]{HORN-1990}, the Rayleigh quotient is bounded by the largest and the smallest eigenvalues of \(\tens{M}_{1,\eta}\). However, here, the smallest non-zero eigenvalue is relevant since we imposed \(\vect{z} \notin \range(\tens{\X})\). That is, the vanishing eigenvalues, which correspond to the kernel space, do not apply. Therefore,
    \begin{equation}
        \eigM_{m+1} \leq \frac{\vect{z}^{\intercal} \tens{M}^2_{1,\eta} \vect{z}}{\vect{z}^{\intercal} \tens{M}_{1,\eta} \vect{z}} \leq \eigM_n. \label{eq:rayleigh}
    \end{equation}

\item\label{item:lambda4} Combining \eqref{eq:mean-mu} and \eqref{eq:rayleigh} with \eqref{eq:der-L-alt-3} yields,
    \begin{equation}
        \left| \frac{\mathrm{d} \ell_{\hat{\sigma}^2(\eta)}(\eta)}{\mathrm{d} \eta} \right| \leq \frac{n-m}{2} \left( \eigM_n - \eigM_{m+1} \right). \label{eq:ddot-l-ineq}
    \end{equation}
    We relate this inequality to the eigenvalues of \(\tens{K}\) in the next step.

\item\label{item:lambda5} Because \(\tens{M}_{1,\eta} = \tens{K}_{\eta}^{-1} \tens{P}_{\eta}\) is the projection of \(\tens{K}_{\eta}^{-1}\) onto a subspace, it can be shown, for instance, by the Courant-Fischer's min-max principle, that the largest eigenvalue of \(\tens{M}_{1,\eta}\) cannot be larger than the largest eigenvalue of \(\tens{K}_{\eta}^{-1} = (\tens{K} + \eta \tens{I})^{-1}\), \ie
    \begin{subequations}
    \begin{equation}
        \eigM_n \leq (\lambda_1 + \eta)^{-1}. \label{eq:ineq-mu-1}
    \end{equation}
        Similarly, the smallest non-zero eigenvalue of \(\tens{M}_{1,\eta}\) cannot be smaller than the smallest eigenvalue of \(\tens{K}_{\eta}^{-1}\), \ie
        \begin{equation}
            \eigM_{m+1} \geq (\lambda_n + \eta)^{-1}. \label{eq:ineq-mu-2}
        \end{equation}
        \end{subequations}
        Combining \eqref{eq:ddot-l-ineq}, \eqref{eq:ineq-mu-1}, and \eqref{eq:ineq-mu-2} concludes \eqref{eq:ldot-bound}.
 
    \item\label{item:lambda6} Finding bound for the second derivative in \eqref{eq:lddot-bound} follows correspondingly. We substitute \(\hat{\sigma}^2(\eta)\) from \Cref{thm:sigma} in \Cref{prop:der-eta} to represent the second derivative of \(\ell\) as,
    \begin{equation}
        \frac{\mathrm{d}^2 \ell_{\hat{\sigma}^2(\eta)}(\eta)}{\mathrm{d} \eta^2} =
        \frac{n-m}{2} \left( \frac{\trace(\tens{M}_{1,\eta}^2)}{n-m} + 
        \left( \frac{\vect{z}^{\intercal} \tens{M}_{1,\eta}^2 \vect{z}}{\vect{z}^{\intercal} \tens{M}_{1,\eta} \vect{z}} \right)^2
        - 2 \frac{\vect{z}^{\intercal} \tens{M}_{1,\eta}^3 \vect{z}}{\vect{z}^{\intercal} \tens{M}_{1,\eta} \vect{z}} \right).
        \label{eq:dd-l-eta-alt}
    \end{equation}
    From the definition of \(\vect{w}\) in \ref{item:lambda3}, the above can be written as,
    \begin{equation}
        \frac{\mathrm{d}^2 \ell_{\hat{\sigma}^2(\eta)}(\eta)}{\mathrm{d} \eta^2} =
        \frac{n-m}{2} \left( \frac{\trace(\tens{M}_{1,\eta}^2)}{n-m} + 
        \left( \frac{\vect{w}^{\intercal} \tens{M}_{1,\eta} \vect{w}}{\| \vect{w} \|^2} \right)^2
        - 2 \frac{\vect{w}^{\intercal} \tens{M}_{1,\eta}^2 \vect{w}}{\| \vect{w} \|^2} \right).
        \label{eq:der2-w}
    \end{equation}

\item\label{item:lambda7} Bounds on the above relation can be obtained as follows. Firstly, recall from \eqref{eq:mu-bar-2} and \eqref{eq:mu-bar-ineq} that,
    \begin{equation}
        \eigM_{m+1}^2 \leq \frac{\trace(\tens{M}_{1,\eta}^2)}{n-m} \leq \eigM_n^2. \label{eq:M2-bound}
    \end{equation}
    Secondly, bounds for the second term in the parenthesis in \eqref{eq:der2-w} is given previously in \eqref{eq:rayleigh}. Thirdly, the last term in the parenthesis in \eqref{eq:der2-w} is the Rayleigh quotient for the matrix \(\tens{M}_{1,\eta}^2\). Thus, with a similar to the argument in \ref{item:lambda3}, we have,
    \begin{equation}
        \eigM_{m+1}^2 \leq \frac{\vect{z}^{\intercal} \tens{M}_{1,\eta}^3 \vect{z}}{\vect{z}^{\intercal} \tens{M}_{1,\eta} \vect{z}} \leq \eigM_n^2. \label{eq:M3-bound}
    \end{equation}
    Combining \eqref{eq:rayleigh}, \eqref{eq:M2-bound}, and \eqref{eq:M3-bound} with \eqref{eq:dd-l-eta-alt} yields,
    \begin{equation}
        \left| \frac{\mathrm{d}^2 \ell_{\hat{\sigma}^2(\eta)}(\eta)}{\mathrm{d} \eta^2} \right| \leq (n-m) \left( \eigM_n^2 - \eigM_{m+1}^2 \right). \label{eq:dd-l-eta-ineq}
    \end{equation}

\item\label{item:lambda8} With a similar argument as in \ref{item:lambda5}, the above inequality can be expressed with the eigenvalues of \(\tens{K}\). Namely, combining \eqref{eq:ineq-mu-1}, \eqref{eq:ineq-mu-2} with \eqref{eq:dd-l-eta-ineq} concludes \eqref{eq:lddot-bound}.
    \end{enumerate}
    

\subsection{Proof of \texorpdfstring{\Cref{lem:M-asym}}{Lemma 10}} \label{prf:M-asym}

    As \(\eta^{-1}\) shrinks, we obtain asymptote of the matrix \(\tens{K}_{\eta}^{-1}\) in \ref{item:M-asym-1}, \(\tens{P}_{\eta}\) in \ref{item:M-asym-2} and \ref{item:M-asym-3}, and \(\tens{M}_{1,\eta}\) in \ref{item:M-asym-4}, as follows.

    \begin{enumerate}[leftmargin=*,align=left,label*=\emph{Step (\roman*).},ref=step (\roman*),wide]
        \item\label{item:M-asym-1} The first three terms of the Neumann series of the matrix \(\tens{K}_{\eta}^{-1}\), as a bounded operator (see \eg \cite[p. 320, Lemma 1]{DAUTRAY-2000}), is,
            \begin{align}
                \tens{K}_{\eta}^{-1} &= \frac{1}{\eta} \left( \tens{I} + \frac{1}{\eta} \tens{K} \right)^{-1} \notag \\
                &= \frac{1}{\eta} \left( \tens{I} - \frac{1}{\eta} \tens{K} + \frac{1}{\eta^2} \tens{K}^2 + \mathcal{O}\left( \| \eta^{-1} \tens{K} \|^3 \right) \right). \label{eq:Kinv-asym}
            \end{align}
            An infinite Neumann series is convergent if \(\| \eta^{-1} \tens{K} \| < 1\), which translates to \(\eta > \| \tens{K} \| = \lambda_n\) using the \(2\)-norm of the symmetric positive-definite matrix \(\tens{K}\). For the truncated series in the above, we impose \(\eta \gg \lambda_n\).

        \item\label{item:M-asym-2} Recall from \eqref{eq:P} and \eqref{eq:Sigma-K} that,
            \begin{equation}
                \tens{P}_{\eta} = \tens{I} - \tens{\X} \left( \tens{\X}^{\intercal} \tens{K}_{\eta}^{-1} \tens{\X} \right)^{-1} \tens{\X}^{\intercal} \tens{K}_{\eta}^{-1}. \label{eq:P2}
            \end{equation}
            Let \( \tens{B} \coloneqq \tens{\X}^{\intercal} \tens{K}_{\eta}^{-1} \tens{\X}\), which is a term in \(\tens{P}_{\eta}\). From \eqref{eq:Kinv-asym} and considering \(\tens{X}\) is full rank, we have,
            \begin{align*}
                \tens{B}^{-1} &= \left( \tens{\X}^{\intercal} \frac{1}{\eta} \left( \tens{I} - \frac{1}{\eta} \tens{K} + \frac{1}{\eta^2} \tens{K}^2 + \mathcal{O}\left(\eta^{-3} \lambda_n^3\right)  \right) \tens{\X} \right)^{-1} \\
                &= \eta \bigg( \tens{I} - \frac{1}{\eta} \underbrace{\left(\tens{\X}^{\intercal} \tens{\X} \right)^{-1} \left( \tens{\X}^{\intercal} \tens{K} \tens{\X} \right)}_{=: \tens{C}_1} + \frac{1}{\eta^2} \underbrace{\left( \tens{\X}^{\intercal} \tens{\X} \right)^{-1} \left( \tens{\X}^{\intercal} \tens{K}^2 \tens{\X} \right)}_{=: \tens{C}_2} + \mathcal{O}\left(\eta^{-3} \lambda_n^3\right) \bigg)^{-1} \left( \tens{\X}^{\intercal} \tens{\X} \right)^{-1}.
            \end{align*}
            Using the identity
            \begin{equation*}
                \left( \tens{I} - \frac{1}{\eta} \tens{C}_1 + \frac{1}{\eta^2} \tens{C}_{2} \right)^{-1} = \tens{I} + \frac{1}{\eta} \tens{C}_1 + \frac{1}{\eta^2}\left(\tens{C}_1^2 - \tens{C}_2 \right) + \mathcal{O}\left(\eta^{-3}\right),
            \end{equation*}
            for the matrices \(\tens{C}_1\) and \(\tens{C}_2\) as defined in the above, we obtain the Neumann series of \(\tens{B}^{-1}\) as,
            \begin{equation*}
                \begin{split}
                    \tens{B}^{-1} =& \eta \bigg[ \left( \tens{\X}^{\intercal} \tens{\X} \right)^{-1} + \frac{1}{\eta} \left( \tens{\X}^{\intercal} \tens{\X}\right)^{-1} \left( \tens{\X}^{\intercal} \tens{K} \tens{\X} \right) \left( \tens{\X}^{\intercal} \tens{\X} \right)^{-1} \\
                    &+ \frac{1}{\eta^2} \left( \left( \left( \tens{\X}^{\intercal} \tens{\X} \right)^{-1} (\tens{\X}^{\intercal} \tens{K} \tens{\X}) \right)^2 \left( \tens{\X}^{\intercal} \tens{\X} \right)^{-1} - \left( \tens{\X}^{\intercal} \tens{\X}\right)^{-1} \left( \tens{\X}^{\intercal} \tens{K}^2 \tens{\X} \right) \left( \tens{\X}^{\intercal} \tens{\X} \right)^{-1} \right) \bigg] \\
                    &+ \mathcal{O}\left( \eta^{-2} \lambda_n^3 \right).
                \end{split}
            \end{equation*}
            We will substitute \(\tens{B}^{-1}\) into \(\tens{P}_{\eta} \) in the next step.

        \item\label{item:M-asym-3} Define the projection matrix \(\tens{Q}_{\perp} \coloneqq \tens{\X} \left( \tens{\X}^{\intercal} \tens{\X} \right)^{-1} \tens{\X}^{\intercal}\), which is the orthogonal complement to the projection matrix \(\tens{Q}\) defined in \eqref{eq:Q}. That is, \(\tens{Q}_{\perp} = \tens{I} - \tens{Q}\). We can show that,
            \begin{equation*}
                \tens{\X} \tens{B}^{-1} \tens{\X}^{\intercal} = \eta \left( \tens{Q}_{\perp} + \frac{1}{\eta} \tens{Q}_{\perp} \tens{K} \tens{Q}_{\perp} + \frac{1}{\eta^2} \Big( \tens{Q}_{\perp} (\tens{K} \tens{Q}_{\perp})^2 - \tens{Q}_{\perp} \tens{K}^2 \tens{Q}_{\perp} \Big) \right) + \mathcal{O}\left(\eta^{-2} \lambda_n^3\right).
            \end{equation*}
            We derive \(\tens{P}_{\eta}\) by substituting the above term and \(\tens{K}_{\eta}^{-1}\) from \eqref{eq:Kinv-asym} into \eqref{eq:P2}. After carrying out the multiplications, we omit terms with order higher than \(\eta^{-2}\), yielding,
            \begin{align*}
                \tens{P}_{\eta} =& \left(\tens{I} - \tens{Q}_{\perp}\right) + \frac{1}{\eta} \left( \tens{Q}_{\perp} \tens{K} - \tens{Q}_{\perp} \tens{K} \tens{Q}_{\perp} \right) \\
                &+ \frac{1}{\eta^2} \Big( -\tens{Q}_{\perp} \tens{K}^2 + (\tens{Q}_{\perp} \tens{K})^2 - \tens{Q}_{\perp} (\tens{K} \tens{Q}_{\perp})^2 + \tens{Q}_{\perp} \tens{K}^2 \tens{Q} \Big) +  \mathcal{O}\left(\eta^{-3} \lambda_n^3\right).
            \end{align*}
            Using \(\tens{Q} = \tens{I} - \tens{Q}_{\perp}\), the above can be further simplified to
            \begin{equation}
                \tens{P}_{\eta} = \tens{Q} + \frac{1}{\eta} \tens{Q}_{\perp} \tens{K} \tens{Q} - \frac{1}{\eta^2} \tens{Q}_{\perp} (\tens{K} \tens{Q})^2 + \mathcal{O}\left(\eta^{-3} \lambda_n^3\right). \label{eq:P3}
            \end{equation}
        \item\label{item:M-asym-4} We calculate \(\tens{M}_{1,\eta} = \tens{K}_{\eta}^{-1} \tens{P}_{\eta}\) from \eqref{eq:Kinv-asym} and \eqref{eq:P3}. Again, we omit terms with order higher than \(\eta^{-2}\) after performing multiplications. Overall, we obtain,
            \begin{equation*}
                    \tens{M}_{1,\eta} = \frac{1}{\eta} \left( \tens{Q} + \frac{1}{\eta} \left( \tens{Q}_{\perp} \tens{K} \tens{Q} - \tens{K} \tens{Q} \right) + \frac{1}{\eta^2} \Big( -\tens{Q}_{\perp} (\tens{K} \tens{Q})^2 - \tens{K} \tens{Q}_{\perp} \tens{K} \tens{Q} + \tens{K}^2 \tens{Q} \Big) + \mathcal{O}\left( \eta^{-3} \lambda_n^3 \right) \right).
            \end{equation*}
            Again, by using \(\tens{Q}_{\perp} = \tens{I} - \tens{Q}\) and factoring terms, the above can be simplified to
            \begin{equation*}
                \tens{M} = \frac{1}{\eta} \tens{Q} \left( \tens{I} - \frac{1}{\eta} \tens{K} \tens{Q} + \frac{1}{\eta^2} (\tens{K} \tens{Q})^2 \right) + \mathcal{O}\left( \eta^{-4} \lambda_n^3 \right).
            \end{equation*}
            Applying \(\tens{N} = \tens{K} \tens{Q}\) concludes \eqref{eq:M-asymptote}.
    \end{enumerate}

\end{appendices}


\phantomsection{}
\addcontentsline{toc}{section}{References}



\bibliographystyle{apalike2}

\bibliography{References}


\end{document}